\documentclass{article}
\usepackage{ijcai21}

\usepackage{times}

\usepackage{soul}
\usepackage{amsthm}
\usepackage{url}
\usepackage[hidelinks]{hyperref}
\usepackage[utf8]{inputenc}
\usepackage{natbib}
\usepackage{amsfonts} 
\usepackage[small]{caption}
\usepackage{graphicx}
\usepackage{amsmath}
\usepackage{booktabs}
\newtheorem{theorem}{Theorem}
\newtheorem{assumption}{Assumption}
\newtheorem{lemma}{Lemma}
\newtheorem{remark}{Remark}

\newcommand{\rank}{\textnormal{rank}}

\newcommand{\diag}{\textnormal{diag}}

\title{A Spectral Approach to Off-Policy Evaluation for POMDPs}

\author{
Yash Nair$^1$\footnote{Contact Author}\and
Nan Jiang$^2$\\
\affiliations
$^1$Harvard University\\
$^2$University of Illinois Urbana-Champaign\\
\emails
yashnair@college.harvard.edu,
nanjiang@illinois.edu
}

\begin{document}

\maketitle

\begin{abstract}
We consider off-policy evaluation (OPE) in Partially Observable Markov Decision Processes, where the evaluation policy depends only on observable variables but the behavior policy depends on latent states \citep{Tennenholtz_Shalit_Mannor_2020}. Prior work on this problem uses a causal identification strategy based on one-step observable proxies of the hidden state, which relies on the invertibility of  certain one-step moment matrices. In this work, we relax this requirement by using spectral methods and extending one-step proxies both into the past and future. We empirically compare our OPE methods to existing ones and demonstrate their improved prediction accuracy and greater generality. Lastly, we derive a separate Importance Sampling (IS) algorithm which relies on rank, distinctness, and positivity conditions, and not on the strict sufficiency conditions of observable trajectories with respect to the reward and hidden-state structure required by \citet{Tennenholtz_Shalit_Mannor_2020}.
\end{abstract}

\section{Introduction}

We consider the problem of estimating the value of a policy---called the \textit{evaluation policy}---given access only to a batch of trajectories obtained from a different policy---called the \textit{behavior policy}---known as the \textit{off-policy evaluation} (OPE) problem. The OPE problem is well-motivated by its application to real-world scenarios in which the deployment of the evaluation policy is potentially too costly for its performance to be determined by directly intervening with it, but rather must be inferred from previously existing data under a different policy. As an example, in healthcare settings, it may be unsafe to directly test a new experimental treatment with potentially unknown harms and side-effects on patients, and thus reasoning based on previous treatment strategies from electronic health records (EHRs) to infer the new treatment's efficacy is necessary. Alternatively, when designing a new personalized educational curriculum, which could potentially have harmful effects on students' learning \citep{mandel2016offline}, the use of historical curriculum data is required. 

OPE has received much attention in the setting of Markov Decision Processes (MDPs) in which the underlying state is fully observable, and both evaluation and behavior policies are (potentially non-deterministic and/or non-stationary) functions of state \citep{sutton2018reinforcement,doubly_robust,precup}. 
We, however, consider the more challenging regime of Partially Observable MDPs (POMDPs), in which one has access only to \textit{observable} trajectories formed under the behavior policy, which do not include underlying latent states (i.e. hidden confounders), and must infer the value of the evaluation policy given only these observables. Following \citet{Tennenholtz_Shalit_Mannor_2020}, we consider the setting in which the behavior policy depends on the latent state, whereas the evaluation policy at any given time step depends only on the observed history up to that time step. This choice is motivated, for example, by the medical treatment problem and autonomous driving scenarios described by \citet{Tennenholtz_Shalit_Mannor_2020}. 

In this setup, \citet{Tennenholtz_Shalit_Mannor_2020} builds on the work of \citet{miao2018identifying} and gives 
a causal effect identification scheme, given two conditionally independent views of the confounder. 
Their estimator uses one-step observable proxies of the hidden state, and crucially relies on the invertibility of certain \textit{one-step} moment matrices, which is a strong assumption and limits the applicability of their method. Furthermore, they use moment matrices whose entries are \textit{conditional} probabilities, which can be difficult to estimate when the condition event is rare and impossible if measure zero.

In this work we extend  \citet{Tennenholtz_Shalit_Mannor_2020} and address the above issues, inspired by techniques for solving similar problems in the literature of spectral learning of  Predictive State Representations (PSRs) and Hidden Markov Models (HMMs) \citep{hsu20121460, boots,insufficient_stat}. Our work extends and improves upon \citet{Tennenholtz_Shalit_Mannor_2020} in several aspects:
\begin{itemize}
\item Our estimators work for a broader family of POMDPs, specifically those in which the observation space is larger than hidden state space, as is the case in many real-world settings.
\item Our estimators estimate matrices of \textit{joint} probabilities rather than \textit{conditionals}, and thus improve generality by not having to condition on probability zero events. 
\item We are able to further relax a crucial assumption in their work regarding the invertibility of one-step moment matrices by working with \textit{multi-step proxies} for the hidden state; for example, our estimator uses the entire extended history preceding the confounder as one of the proxies, rather than just the previous time step.  
\end{itemize}

While we are able to successfully use multi-step histories as the past proxy, we show that there is difficulty in directly using extended futures as the future proxy (which would further relax rank assumptions). However, by employing the eigendecomposition technique of \citet{kuroki}, we are able---with minimal additional assumptions---to give an estimator which also incorporates extended futures and thus applies even more generally. 

We then empirically validate our estimators, comparing them to prior estimators \citep{Tennenholtz_Shalit_Mannor_2020} and demonstrate their superior predictive accuracy. Finally, by again employing the eigendecomposition technique of \citet{kuroki}, we provide a separate Importance Sampling (IS) algorithm which requires only rank, positivity, and distinctness conditions on certain probability matrices, and not the sufficiency assumptions required by the IS algorithm of \citet{Tennenholtz_Shalit_Mannor_2020}.

\section{Related Work}
\paragraph{OPE in POMDPs} Most closely related to our work is that of \citet{Tennenholtz_Shalit_Mannor_2020}, who consider the finite-horizon OPE problem in tabular POMDPs with equal-sized hidden state and observation spaces by extending  \citet{miao2018identifying} and writing the probability of seeing reward $r_t$, in terms of observable moment matrices under the behavior measure. These matrices have dimensions which scale with the size of the observation space, $\mathcal{Z}$, and are required to be invertible. They further develop a Decoupled POMDP model which factors the state space, $\mathcal{U}$, into both observed and unobserved variables, allowing for a modified OPE algorithm whose moment matrices scale with $|\mathcal{U}|$, which is often much smaller than $|\mathcal{Z}|$. Our work mitigates difficulty in the more-general POMDP setting by allowing for rich observation spaces with greater cardinality than state space, requiring that the observable matrices have rank $|\mathcal{U}|$, rather than  $|\mathcal{Z}|$, and further relaxes assumptions and attains superior performance by extending both histories and futures. Additionally, \citet{oberst} consider counterfactuals, asking about the outcome of a counterfactual trajectory under the evaluation policy, given the observable trajectory's outcome under the behavior policy. Their work differs from ours in that, just as in \citet{Tennenholtz_Shalit_Mannor_2020}, we do not ask about the evaluation policy's counterfactual performance compared to the behavior policy's, but rather are concerned with determining the value of the evaluation policy given observable trajectories under the behavior.

\paragraph{OPE with unobserved confounding} The OPE problem has also been studied in the setting of the MDP with Unmeasured Confounding (MDPUC) \citep{Zhang2016MarkovDP}. The MDPUC framework is orthogonal to POMDPs, with the local confounders assumed to be i.i.d. across time steps. The OPE problem is considered in an infinite-horizon, non-tabular MDPUC by \citet{bennett2020offpolicy}. Critically, however, their work assumes access to a latent variable model for confounders, which we do not assume. Specifically, in the spirit of \citet{Tennenholtz_Shalit_Mannor_2020} and \citet{miao2018identifying}, our initial relaxations of rank assumptions do not even depend on non-parametric identification of the latent variable model. Our extended-future estimator and IS results, however, do rest on the identification of aspects of the latent variable model, but obtain estimates from observable data instead of assuming access to a model.

\paragraph{Causal effect identification} Our work is closely related to unobserved confounding in the causal inference literature, and, specifically, the use of negative controls to minimize confounding bias. Of particular relevance are \citet{miao2018identifying}, \citet{kuroki} and \citet{greenland}, which all consider the static problem of identifying causal effects, assuming multiple observable proxies of the latent confounder. \citet{kuroki} first identify the confounder's error mechanism via eigenvalue analysis and then use the matrix adjustment method \citep{greenland} to determine the causal effect, while \citet{miao2018identifying} directly determine the causal effect without identifying any of the aspects of the confounding model, allowing for weaker assumptions. Our work extends the latter's result to the setting of POMDPs just as in \citet{Tennenholtz_Shalit_Mannor_2020}; our spectral result in the setting of POMDPs also applies to the static setting and hence relaxes rank assumptions in \citet{miao2018identifying}. Additionally, we extend \citet{kuroki}'s analysis to identify causal effects of actions on futures in our extended-future estimator as well as to recover the behavior policy and reward distribution in our IS algorithm.

\section{Preliminaries}
\paragraph{POMDPs} A POMDP is a $7$-tuple $\langle \mathcal{U}, \mathcal{Z}, \mathcal{A}, O, T, R, H\rangle$, where $\mathcal{U}$ and $\mathcal{Z}$ denote the hidden state and observation spaces, respectively, $\mathcal{A}$ denotes the set of actions, $O: \mathcal{U}\times \mathcal{Z} \rightarrow [0,1]$ gives observation emission probabilities, with $O(z|u) := \Pr[z|u]$, $T: \mathcal{U}\times \mathcal{U}\times \mathcal{A}$ governs the hidden state transition dynamics, with $T(u'| u, a) := \Pr[u'|u,a]$, $R$ is a (potentially non-deterministic) function of the hidden state and action giving the reward which is supported in the set $\mathcal{R}$, and $H$ denotes the horizon. At each time step, the agent selects an action, $a$, transitions under $T$, to a new hidden state, $u'$, based on the current hidden state, $u$, and $a$ and then observes an observation, $z$, drawn according to $O(\cdot|u')$, and reward, $r$, drawn from its conditional distribution over the hidden state and action. In this work, we restrict ourselves to finite $\mathcal{U}$, $\mathcal{Z}$, and $\mathcal{A}$. Furthermore, we assume that $|\mathcal{U}| \leq |\mathcal{Z}|$, reflecting the fact that in most real-world settings, the hidden state space is often much smaller than the rich observation space.

Using the notation of \citet{Tennenholtz_Shalit_Mannor_2020}, a trajectory $\tau$ of length $t$ denotes a sequence $(u_0, z_0, a_0, \ldots, u_{t}, z_{t}, a_{t})$ while $\tau^o$ denotes the observable trajectory $(z_0, a_0, \ldots, z_t, a_t)$. We use $\mathcal{T}_t$ to denote the set of trajectories of length $t$, and, correspondingly, $\mathcal{T}^o_t$ to denote the set of observable trajectories of length $t$. Additionally, an observable history of length $t$ is given by a sequence of the form $(z_0, a_0, \ldots, a_{t-1}, z_t)$, and is denoted by $h^o_t$. Finally, we assume the existence of an observation $z_{-1}$ preceding the initial time step which is conditionally independent of $z_0$ and $a_0$ given $u_0$.

\paragraph{Policies} Let $\pi_b$ and $\pi_e$ denote the behavior and evaluation policies, respectively. We consider the setting in which the behavior policy at time step $t$, $\pi_b^{(t)}$, depends only on the hidden state, $u$, whereas the evaluation policy at time $t$, $\pi_e^{(t)}$, depends on the observable history, $h^o_t$. In particular, $\pi_b^{(t)}(a_t|u_t)$ denotes the probability that an agent (with access to hidden state) following $\pi_b$ selects action $a_t$ at time step $t$ given that it is at $u_t$, and $\pi_e^{(t)}(a_t|h^o_t)$ denotes the probability that an agent (without access to hidden state) following $\pi_e$ selects action $a_t$ at time step $t$ given that it has seen the observable history $h^o_t$. Additionally, we consider the finite-horizon undiscounted setting so that for any policy $\pi$ (either depending on observable histories or hidden state), its value $v_H(\pi)$ is given by $\mathbb{E}_\pi\left[\sum_{t=0}^Hr_t\right]$, the expected sum of rewards gotten by following $\pi$. Finally, let $P^b$ and $P^e$ denote measures over trajectories induced by the behavior and evaluation policies, respectively.

In addition to the above, we use a double vertical bar notation to indicate intervening \citep{boots}. For example, $P(z_0, \ldots, z_t||a_0,\ldots, a_{t-1})$ denotes the probability of seeing observations $(z_0, \ldots, z_t)$, given that we intervened with actions $a_0,\ldots, a_{t-1}$. Given a policy, $\pi$, $P^\pi(z_0, \ldots, z_t|a_0, \ldots, a_{t-1})$ is in general not equal to $P(z_0, \ldots, z_t||a_0,\ldots, a_{t-1})$, as in the former, actions are random variables, which may convey information about hidden state, whereas in the latter, they are non-random interventions. This can also be written using the \textit{do} operator of \citet{pearl}.

We use matrix notations similar to \citet{Tennenholtz_Shalit_Mannor_2020}. Let random variables $\mathbf{x}, \mathbf{y}, \mathbf{z}$ be supported in $X = \{x_1, \ldots, x_{n_1}\}, Y = \{y_1, \ldots, y_{n_2}\}, Z = \{z_1, \ldots, z_{n_3}\}$. Then $P(y|X), P(X), P(Y,x, Z)$ denote a row vector, column vector, and matrix, with $(P(y|X))_i = \Pr[y|x_i], (P(X))_i = \Pr[x_i], (P(Y,x,Z))_{i,j} = \Pr[y_i,x, z_j]$. For joint probabilities, the first (second) set listed denotes vectorization across rows (columns); for conditionals, the set listed before (after) the conditioning bar indicates vectorization across rows (columns).

\paragraph{Problem} Formally, the OPE problem asks, given a batch of observable trajectories collected under the behavior policy, to estimate the value of the evaluation policy. In other words, we must estimate $v_H(\pi_e)$ given only observable trajectories $\tau^o$ drawn from the behavior measure $P^b$.

\subsection{One-Step Proxies}
We now briefly review the OPE algorithm of \citet{Tennenholtz_Shalit_Mannor_2020}, who approach the problem by estimating $P^e(r_t)$, which can then be used to approximate $v_H(\pi_e)$ by estimating the expected reward at each timestep. (In Sections~\ref{sec:relax} and \ref{sec:multi-future} we also focus on estimating $P^e(r_t)$.) They make the following assumptions on observable probability matrices.

\begin{assumption}\label{invertibility}
$P^b(Z_i|a_i, Z_{i-1})$ is invertible $\forall i \leq H, a_i \in \mathcal{A}$.
\end{assumption}

Defining $\Pi^e_t(\tau^o) = \prod_{i=0}^t\pi_e^{(i)}(a_i|h^o_i)$, they give the following estimator:
\begin{theorem}\label{original}
Under Assumption \ref{invertibility}, $P^e(r_t)$ can be written as \begin{multline} \label{eq:5}
\sum_{\tau^o \in \mathcal{T}^o_t} \Pi_e(\tau^o)P^b(r_t, z_t|a_t, Z_{t-1}) \\ \cdot \prod_{i=t-1}^0\left(P^b(Z_{i+1}|a_{i+1}, Z_i)^{-1}P^b(Z_{i+1}, z_i|a_i, Z_{i-1})\right) \\ \cdot P^b(Z_0|a_0, Z_{-1})^{-1}P^b(Z_0).
\end{multline}
\end{theorem}

We give a sketch of their proof (for details, see \citet{Tennenholtz_Shalit_Mannor_2020}).
\begin{proof}[Proof Sketch]
Note that \begin{equation}\label{interventions}P^e(r_t) = \sum_{\tau^o \in \mathcal{T}^o_t}\Pi^e_t(\tau^o)P^b(r_t, z_t, z_{t-1}, \ldots, z_0||a_0, \ldots, a_t).\end{equation} The analysis rests on non-parametric identification of the causal effect of $a_i$ on the outcome $(u_{i+1}, z_i)$ given confounded proxies $z_{i-1}$ and $z_i$ at each time step. This is done by first noting that \begin{multline} \label{eq:1}
P(r_t, z_t, z_{t-1}, \ldots, z_0||a_0, \ldots, a_t) \\ = P^b(r_t, z_t|a_t, U_t)\left(\prod_{i=t-1}^0P^b(U_{i+1}, z_i|a_i, U_i)\right)P^b(U_0),
\end{multline} and extending off of the static causal identification setting, to show that \begin{multline} \label{eq:2}
P^b(U_{i+1}, z_i|a_i, U_i)P^b(U_i, z_{i-1}|a_{i-1}, U_{i-1}) \\ = P^b(U_{i+1}, z_i|a_i, Z_{i-1})P^b(Z_i|a_i, Z_{i-1})^{-1} \\ \cdot P^b(Z_i, z_{i-1}|a_{i-1}, Z_{i-2})P^b(Z_{i-1}|a_{i-1}, Z_{i-2})^{-1} \\ \cdot P^b(Z_{i-1}|a_{i-1}, U_{i-1})
\end{multline} and \begin{multline} \label{eq:3}
P^b(r_t, z_{t}|a_{t}, U_{t})P^b(U_{t}, z_{t-1}|a_{t-1}, U_{t-1}) \\ = P^b(r_t, z_{t}|a_{t}, Z_{t-1})P^b(Z_{t}|a_{t}, Z_{t-1})^{-1} \\ \cdot P^b(Z_{t}, z_{t-1}|a_{t-1}, Z_{t-2})P^b(Z_{t-1}|a_{t-1}, Z_{t-2})^{-1} \\ \cdot P^b(Z_{t-1}|a_{t-1}, U_{t-1})
\end{multline} while also observing that \begin{multline} \label{eq:4}
    P^b(Z_i|a_i, U_i)P^b(U_i, z_{i-1}|a_{i-1}, Z_{i-2}) \\= P^b(Z_i, z_{i-1}|a_{i-1}, Z_{i-2}).
\end{multline} The result then follows from these four core identities and induction.
\end{proof}

In the language of causal inference, the proof is made possible by, at each time step $i$, statically viewing $z_{i-1}$ as an observable proxy emitted by the latent state $u_i$, giving the causal diagram in Figure \ref{fig:scms}(a) illustrating that the proxies $z_i$ and $z_{i-1}$ are conditionally independent given $u_i$, allowing for the application of the identification scheme of \citet{miao2018identifying}. 

\begin{figure}
    \centering
    \includegraphics[scale=0.8]{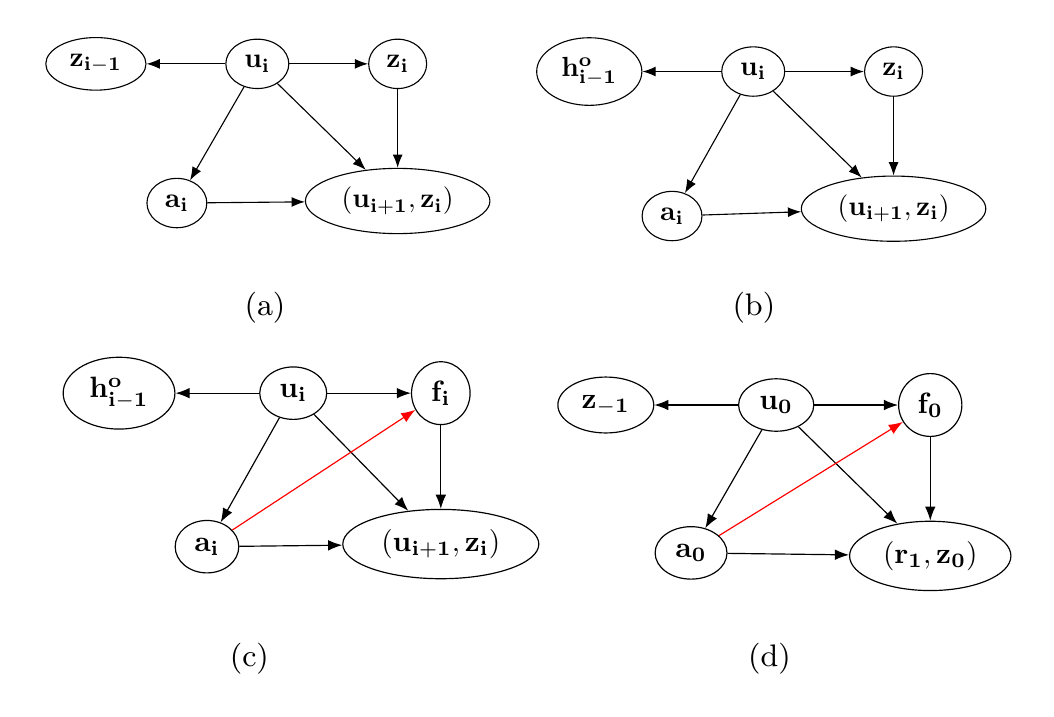}
    \caption{Causal diagrams for proxies of the confounder $\mathbf{u_i}$: one-step proxies (a), multi-step histories (b), multi-step histories, multi-step futures (c), and multi-step futures with reward-observation outcome (d).}
    \label{fig:scms}
\end{figure}

\section{Relaxing Rank Assumptions} 
\label{sec:relax}
The proof of Theorem \ref{original} requires that $|\mathcal{U}| = |\mathcal{Z}|$ which is unrealistic in most real-world settings, where the hidden state space is often much smaller than the rich observation space. Inspired by spectral learning of PSRs and POMDPs \citep{boots,insufficient_stat,hsu20121460}, we show that we can allow for the more general assumption that $|\mathcal{U}| \leq |\mathcal{Z}|$ by requiring that $\rank(P^b(Z_{i}|a_{i}, Z_{i-1})) = |\mathcal{U}|$. We then further relax rank assumptions by using the entire observable history as a proxy, rather than just the observation from the previous time step. The results in the following two sections hold using matrices of conditional probabilities, however we state them using joint probabilities as doing so is more general by allowing for probability zero events which would otherwise be conditioned on. Additionally, all matrix rank assumptions (\ref{rank}--\ref{future-rank}) are stated with equality rather than inequality (i.e., rank equal to $|\mathcal{U}|$ rather than at least $|\mathcal{U}|$) as all such matrices have rank upper bounded by $|\mathcal{U}|$; assuming rank $|\mathcal{U}|$ should be thought of as a ``full rank'' assumption. 

\subsection{Spectral Relaxation of One-Step Proxies}
\begin{assumption}\label{rank}
$\rank(P^b(Z_i,a_i, Z_{i-1})) = |\mathcal{U}|, \forall i \leq H, a_i \in \mathcal{A}$. 
\end{assumption}
In contrast to Theorem~\ref{invertibility} which requires $|\mathcal{Z}|=|\mathcal{U}|$ and $P^b(Z_i|a_i, Z_{i-1})$ to be invertible, here we only require $|\mathcal{Z}| \ge |\mathcal{U}|$ and the matrix to have the maximal possible rank $|\mathcal{U}|$. Moreover, we place the maximum rank condition on the joint probability matrix $P^b(Z_i,a_i, Z_{i-1})$ instead of its conditional version in Assumption~\ref{invertibility}, which is strictly more general: Assumption~\ref{rank} automatically implies that $\rank(P^b(Z_i|a_i, Z_{i-1})) = |\mathcal{U}|$ as $P^b(Z_i,a_i, Z_{i-1}) = P^b(Z_i|a_i, Z_{i-1})\text{diag}(P^b(Z_{i-1},a_i))$ (for matrices, $A$, $B$, $\rank(AB) \leq \min(\rank(A), \rank(B))$). 

With this Assumption, we are able to derive the following estimator of $P^e(r_t)$.
\begin{theorem}\label{one-step-thm}
Under Assumption \ref{rank}, there exists matrices $M_{i,a_i} \in \mathbb{R}^{|\mathcal{U}|\times |\mathcal{Z}|}$ and $M_{i, a_i}' := (M_{i,a_i}P(Z_i,a_i, Z_{i-1}))^+$ such that $P^e(r_t)$ is equal to \begin{multline}
    \sum_{\tau^o \in \mathcal{T}^o_t}\Pi_e(\tau^o)P^b(r_t, z_t,a_t, Z_{t-1})M'_{t,a_t} \\ \cdot \prod_{i=t-1}^0\left(M_{i+1,a_{i+1}}P^b(Z_{i+1}, z_i,a_i, Z_{i-1})M'_{i,a_i}\right)M_{0,a_0}P^b(Z_0)
\end{multline}
\end{theorem}
\begin{remark}[Choice of $M_{i, a_i}$] \label{remark1}
The matrices $M_{i,a_i}$ must be chosen so as to satisfy the non-degenerate conditions that $\rank(M_{i,a_i}) = |\mathcal{U}|$ and $\rank(M_{i,a_i}P(Z_i,a_i, Z_{i-1})) = |\mathcal{U}|$. One such choice is that $M_{i,a_i}$ is the top $|\mathcal{U}|$ left singular vectors of $P(Z_i,a_i, Z_{i-1})$. Alternatively, such a non-degenerate $M_{i,a_i}$ may be drawn from any continuous probability distribution over matrices of the appropriate dimension, almost surely.
\end{remark}

The proof of Theorem \ref{one-step-thm} relies on deriving the following analogs to equations \eqref{eq:2} -- \eqref{eq:3} and using equations \eqref{eq:1} and \eqref{eq:4} to similarly proceed by induction (see Appendix \ref{one-step-proof} for proof).

\begin{lemma}\label{one-step-lemma}
Given Assumption \ref{rank}, define $M_{i,a_i}$ and $M'_{i,a_i}$ as above. Then \begin{multline}\label{eq:2_analog}
    P^b(U_{i+1}, z_i|a_i, U_i)P^b(U_i, z_{i-1}|a_{i-1}, U_{i-1}) \\ = P^b(U_{i+1}, z_i,a_i, Z_{i-1})M'_{i,a_i}M_{i,a_i}P^b(Z_i, z_{i-1},a_{i-1}, Z_{i-2}) \\ \cdot M'_{i-1,a_{i-1}} M_{i-1,a_{i-1}}P^b(Z_{i-1}|a_{i-1}, U_{i-1})
\end{multline} and \begin{multline}\label{eq:3_analog}
    P^b(r_{t}, z_t|a_t, U_t)P^b(U_t, z_{t-1}|a_{t-1}, U_{t-1}) \\ = P^b(r_t, z_t,a_t, Z_{t-1})M'_{t,a_t}M_{t,a_t}P^b(Z_t, z_{t-1},a_{t-1}, Z_{t-2}) \\ \cdot M'_{t-1,a_{t-1}} M_{t-1,a_{t-1}}P^b(Z_{t-1}|a_{t-1}, U_{t-1})
\end{multline}
\end{lemma}

\subsection{Multi-Step Histories}
We now show that, by extending histories, even weaker rank assumptions can be made. Specifically, letting $\mathcal{H}^o_{t}$ denote the set of observable histories of length $t$, the matrices, $P^b(Z_i,a_i, Z_{i-1})$ and $P^b(Z_{i+1}, z_i,a_i, Z_{i-1})$, in Lemma \ref{one-step-lemma} and Theorem \ref{one-step-thm} can be replaced with $P^b(Z_i,a_i, \mathcal{H}^o_{i-1})$ and $P^b(Z_{i+1}, z_i,a_i, \mathcal{H}^o_{i-1})$, respectively. We make the following rank assumption:

\begin{assumption}\label{history-rank}
$\rank(P^b(Z_i,a_i, \mathcal{H}^o_{i-1})) = |\mathcal{U}|, \forall i \leq H, a_i \in \mathcal{A}$.
\end{assumption}
This assumption is strictly weaker than our previous Assumption~\ref{rank}. This is because $P^b(Z_i,a_i, Z_{i-1})$ can be written as $P^b(Z_i,a_i, \mathcal{H}^o_{i-1})J_i$, where $J_i \in \mathbb{R}^{|\mathcal{H}^o_{i-1}|\times|\mathcal{Z}_{i-1}|}$ marginalizes out all parts of the history prior to timestep $i-1$ for each $z_{i-1}$. Since $P^b(Z_i,a_i, Z_{i-1})$ is a projection of $P^b(Z_i,a_i, \mathcal{H}^o_{i-1})$, the rank of the latter cannot be smaller than that of the former. We then obtain the following:
\begin{theorem}\label{history-thm}
Under Assumption \ref{history-rank}, there exists $M_{i,a_i}$ and $M'_{i,a_i} := (M_{i,a_i}P^b(Z_i,a_i, \mathcal{H}^o_{i-1}))^+$, such that the probability $P^e(r_t)$ is equal to \begin{multline}
    \sum_{\tau^o \in \mathcal{T}^o_t}\Pi_e(\tau^o)P^b(r_t, z_t,a_t, \mathcal{H}^o_{t-1})M'_{t,a_t} \\ \cdot \prod_{i=t-1}^0\left(M_{i+1,a_{i+1}}P^b(Z_{i+1}, z_i,a_i, \mathcal{H}^o_{i-1})M'_{i,a_i}\right)M_{0,a_0}P^b(Z_0)
\end{multline}
\end{theorem}
\begin{remark}[Choice of $M_{i, a_i}$]\label{remark2}
Similar to Remark \ref{remark1}, the matrices $M_{i,a_i}$ must be chosen so that $\rank(M_{i,a_i}) = \rank(M_{i,a_i}P(Z_i,a_i, \mathcal{H}^o_{i-1}))= |\mathcal{U}|$. Again, two viable choices are for $M_{i,a_i}$ to be the top $|\mathcal{U}|$ left singular vectors of $P(Z_i,a_i, \mathcal{H}^o_{i-1})$ or to draw $M_{i,a_i}$ from a continuous probability distribution. Furthermore, $M_{i,a_i}$ can be chosen to first marginalize out multi-step histories $\mathcal{H}^o_{i-1}$ into one-step pasts $\mathcal{Z}_{i-1}$, and then act as a rank $|\mathcal{U}|$ projection matrix, as in Theorem \ref{one-step-thm}. This choice of $M_{i,a_i}$ illustrates that Theorem \ref{history-thm} is a strict generalization of Theorem \ref{one-step-thm} as the latter is recoverable with this alternate choice of $M_{i,a_i}$, which first performs marginalization.
\end{remark}

We give a proof of Theorem \ref{history-thm} in Appendix \ref{hist-proof}, but, at a high level, Figures \ref{fig:scms}(a) and (b) show the static causal structure is unchanged when extending the length of observable histories, allowing for essentially the same analysis as in the one-step case.
We also remark that Theorem \ref{history-thm} holds using conditional probability matrices $P^b(Z_i|a_i, \mathcal{H}^o_{i-1})$ and $P^b(Z_{i+1}, z_i|a_i, \mathcal{H}^o_{i-1})$ in place of $P^b(Z_i,a_i, \mathcal{H}^o_{i-1})$ and $P^b(Z_{i+1}, z_i,a_i, \mathcal{H}^o_{i-1})$, respectively. An analogous statement applies to Theorem \ref{one-step-thm}.

\section{Multi-Step Futures} \label{sec:multi-future}
In the previous section we relaxed the rank assumptions in prior work by replacing the one-step observation with \textit{histories} of observations as proxies. It is then very natural to ask: can we make a similar extension to the future? Indeed, in the PSR literature that inspire our work, it is common to use moment matrices about multi-step histories and futures \citep{boots} to relax the rank assumptions of earlier works that estimate one-step moment matrices \citep{hsu20121460}.

To this end, our goal in this section is to extend futures and replace the matrices $P^b(Z_i,a_i, \mathcal{H}^o_{i-1})$ and $P^b(Z_{i+1}, z_i,a_i, \mathcal{H}^o_{i-1})$ with matrices of the form $P^b(\mathcal{F}^o_i,a_i, \mathcal{H}^o_{i-1})$ and $P^b(\mathcal{F}^o_{i+1}, z_i,a_i, \mathcal{H}^o_{i-1})$, respectively ($\mathcal{F}^o_i$ denotes a set of futures at timestep $i$, containing subtrajectories $f_i$ of the form $(z_i, z_{i+1}, a_{i+1}, \ldots, z_H, a_H)$), so as to further relax rank assumptions by only assuming that $\rank(P^b(\mathcal{F}^o_i, a_i, \mathcal{H}^o_{i-1})) = |\mathcal{U}|$. Such an assumption would imply Assumption \ref{history-rank} and thus be even weaker than that made in the previous section.

Unfortunately, such a goal is not immediately possible to achieve via identification of $P^b(r_t, z_t, z_{t-1}, \ldots, z_0||a_0, \ldots, a_t)$ using previous techniques. 
 The primary difficulty, as indicated by the red edge in Figure \ref{fig:scms}(c), is that, when extending futures, the treatment, $a_i$, \textit{does} have a causal effect on the future $f_i$, rendering it a post-treatment variable rather than a negative control as was the case with $z_i$, in Figures \ref{fig:scms}(a) and (b), on which $a_i$ has no causal effect. 


While our approach cannot use extended futures as a proxy for the confounder under rank $|\mathcal{U}|$ assumptions on $P^b(\mathcal{F}^o_i,a_i, \mathcal{H}^o_{i-1})$ alone, we show that it is in fact possible to perform OPE by working with conditional moment matrices $P^b(\mathcal{F}^o_i|a_i, \mathcal{H}^o_{i-1})$ and making relatively weak positivity and distinctness assumptions about the reward distributions:

\begin{assumption}\label{future-rank}
$\rank(P^b(\mathcal{F}^o_i|a_i, \mathcal{H}^o_{i-1})) = |\mathcal{U}|, \forall i \leq H, a_i \in \mathcal{A}$.
\end{assumption}
\begin{assumption}\label{reward-pos-distinct}
For each $a_i \in \mathcal{A}$, there exists $r_i \in \mathcal{R}$ such that the probabilities $P^b(r_i|a_i, u_i), u_i \in \mathcal{U}$ are all distinct.
\end{assumption}

Intuitively, Assumption \ref{reward-pos-distinct} requires that the reward at each timestep---another post-treatment variable---contains enough information about hidden state so as to enable its use, in tandem with futures, to perform causal effect identification (see Appendix \ref{future-proof} for proof):

\begin{theorem}\label{future_thm}
Under Assumptions \ref{future-rank}--\ref{reward-pos-distinct}, there exist $M_{i,a_i}$ and $M'_{i,a_i}$ (chosen to be non-degenerate w.r.t. $P^b(\mathcal{F}^o_i|a_i, \mathcal{H}^o_{i-1})$, analogously to Remarks \ref{remark1}--\ref{remark2}), such that 
$P^e(r_t)$ equals \begin{multline}\label{future-eq}
    \sum_{\tau^o \in \mathcal{T}^o_t}\Pi_e(\tau^o)P^b(r_t, z_t|a_t, \mathcal{H}^o_{t-1})M'_{t,a_t} \cdot \prod_{i=t-1}^0\Big(M_{i+1,a_{i+1}}\\ \cdot P^b(\mathcal{F}^o_{i+1}, z_i|a_i, \mathcal{H}^o_{i-1}||a_{i+1})M'_{i,a_i}\Big)M_{0,a_0}P^b(Z_0),
\end{multline} where the matrices $P^b(\mathcal{F}^o_{i+1}, z_i|a_i, \mathcal{H}^o_{i-1}||a_{i+1})$ are identifiable.
\end{theorem}

The proof of equation \eqref{future-eq} in Theorem \ref{future_thm} is similar to those of Theorems \ref{one-step-thm} and \ref{history-thm}. The key difference with the previous approaches, however, is the identification of the matrices $P^b(\mathcal{F}^o_{i+1}, z_i|a_i, \mathcal{H}^o_{i-1}||a_{i+1})$, which relies on the diagonalization strategy of \citet{kuroki}. Essentially, by diagonalizing products of observable moment matrices, we are able to identify the unobserved matrices $P^b(\mathcal{F}^o_{i+1}|a_{i+1}, U_{i+1})$ and $P^b(U_{i+1}, z_i|a_{i}, \mathcal{H}^o_{i-1})$ which allows for identification of \begin{multline*}
    P^b(\mathcal{F}^o_{i+1}, z_i|a_i, \mathcal{H}^o_{i-1}||a_{i+1}) \\ = P^b(\mathcal{F}^o_{i+1}|a_{i+1}, U_{i+1})P^b(U_{i+1}, z_i|a_{i}, \mathcal{H}^o_{i-1})
\end{multline*}

\section{Experiments}\label{sec:experiments}
We now empirically show that our methods extend the range of environments in which we can successfully perform accurate OPE. In particular, we show that: our history-extension approach of Theorem \ref{history-thm} is more accurate than that of Theorem \ref{original}, the baseline estimator of \citet{Tennenholtz_Shalit_Mannor_2020}, even in environments satisfying Assumption \ref{invertibility}; using our extended-history estimator significantly outperforms the baseline in environments under which Assumption \ref{invertibility} does not hold; and that while Theorem \ref{history-thm}'s estimator \textit{still} performs well on environments in which Assumption \ref{history-rank} doesn't hold, the extended-future estimator from Theorem \ref{future_thm} can produce even more accurate results.

\paragraph{Setup} For each of our experiments, we conduct $50$ independent trials in POMDPs with $|\mathcal{U}| = |\mathcal{Z}| = |\mathcal{A}| = 2$, $\mathcal{R} = \{0,1\}$, and $H=4$. For the first experiment, comparing the estimators of Theorems \ref{original} and \ref{history-thm}, the probability parameters of the POMDP and policies are drawn uniformly at random from $[0,1]$ each trial, whereas in the two latter experiments, the parameters are chosen specifically, so as to violate Assumptions \ref{invertibility} and \ref{history-rank}, and remain constant across trials. We note however, the extended-future estimator of Theorem \ref{future_thm} can be quite sensitive to ill-conditioned matrices, and so the environment for the third experiment was chosen somewhat carefully, through adaptive search, so as to yield an environment with well-conditioned matrices. For the first experiment, during each trial for estimator $E$, we compute, at each timestep, the residual $|\hat{P}^e(r_t = 1) - P^e(r_t=1)|$, where $\hat{P}^e(r_t = 1)$ is $E$'s estimate of $P^e(r_t=1)$ after seeing $n$ trajectories for various values of $n$. For the latter experiments, the results of a given trial for estimator $E$ consist of $E$'s probability estimates of $P^e(r_0 = 1), \ldots, P^e(r_H = 1)$ after seeing $n$ trajectories, for various values of $n$. The final results for each experiment are then averaged over the $50$ trials and presented with standard error bars. 

\begin{figure}
    \centering
    \includegraphics[scale=0.52]{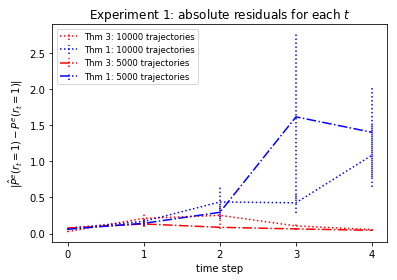}
    \caption{Average absolute residuals $|\hat{P}^e(r_t) -P^e(r_t)|$ over randomly generated POMDP environments which satisfy Assumption \ref{invertibility}.}
    \label{fig:exp_1}
\end{figure}

\begin{figure}
    \centering
    \includegraphics[scale=0.52]{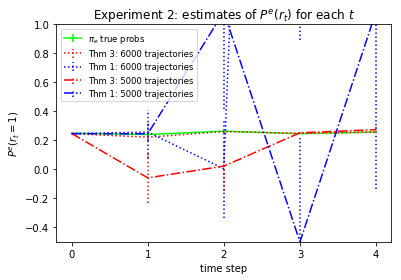}
    \caption{Estimates of $P^e(r_t)$ in POMDP environment which violates Assumption \ref{invertibility}.}
    \label{fig:hist_exp}
\end{figure}

\paragraph{Results} Figure \ref{fig:exp_1} shows that even in environments in which Assumption \ref{invertibility} \textit{does} hold, the quality of the baseline estimator of Theorem \ref{original} degrades substantially, on average, as $t$ grows for both $5000$ and $10{,}000$ trajectories. On the other hand, our history-extended estimator of Theorem \ref{history-thm} maintains accurate estimates of $P^e(r_t=1)$ even for larger values of $t$, and is in general, much more accurate than the baseline at all timesteps.

Figure \ref{fig:hist_exp} extends these results to the second experiment in which we compare our estimator from Theorem \ref{history-thm} to the baseline of Theorem \ref{original} when its Assumption \ref{invertibility} is violated. At $5000$ trajectories, both estimators incur a noticeable degree of error at multiple timesteps, although the errors of the baseline estimator are significantly worse. However, at $6000$ samples, the extended-history estimator of Theorem \ref{history-thm} incurs negligible error in its estimates at \textit{all} timesteps, especially as compared to the baseline which still performs extremely poorly, often producing estimates far outside of $[0,1]$.

Lastly, the robustness of the extended-history estimator of Theorem \ref{history-thm} to an environment in which its Assumption \ref{history-rank} is violated is illustrated in Figure \ref{fig:future_exp}. While Theorem \ref{history-thm}'s estimator is somewhat biased for both $500{,}000$ and $1{,}000{,}000$ sampled trajectories, the estimates are, in absolute terms, relatively accurate, especially when compared to the errors incurred by the baseline estimator of Theorem \ref{original} when its key Assumption \ref{invertibility} was violated in the second experiment. In Appendix \ref{exper-details}, we further show that, in fact, Theorem \ref{history-thm}'s estimator is still highly accurate on average across environments that violate its Assumption \ref{history-rank}. Figure \ref{fig:future_exp} also shows that the extended-future estimator of Theorem \ref{future_thm}---which does not require Assumption \ref{history-rank}---produces more accurate results.

\begin{figure}
    \centering
    \includegraphics[scale=0.52]{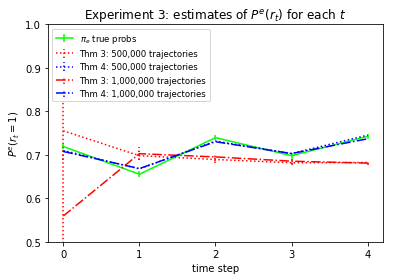}
    \caption{Estimates of $P^e(r_t)$ in POMDP environment which violates Assumption \ref{history-rank}.}
    \label{fig:future_exp}
\end{figure}

\section{Importance Sampling}
Importance Sampling \citep{precup} is another family of standard OPE methods, which complements the so-called ``direct methods'' that either estimate a model or approximate the Bellman equations \citep{dudik2011doubly, doubly_robust}. (Our approach in previous sections can be roughly categorized as a direct method.) 
\citet{Tennenholtz_Shalit_Mannor_2020} illustrate the difficulty of IS in POMDPs by showing that its immediate extension to POMDPs is neither unbiased nor consistent unless strong sufficiency assumptions are made. By extending the eigendecomposition technique of \citet{kuroki}, we show, that OPE in POMDPs \textit{is} amenable to a consistent IS estimator under much weaker assumptions involving only rank $|\mathcal{U}|$, distinctness, and positivity conditions. 

In order to satisfy certain rank assumptions of our analysis we require the following assumptions on the relative sizes of hidden state, action and reward spaces:

\begin{assumption}\label{necessary}
$|\mathcal{A}| \geq |\mathcal{U}|$ and $|\mathcal{R}| \geq |\mathcal{U}|$
\end{assumption}


Under rank $|\mathcal{U}|$ conditions on certain moment matrices, we are then able to identify the behavior policy probabilities $\pi_b(a_i|U_i)$ via eigendecompositions. 
We, additionally, make a weak assumption on the reward distribution so that the causal structure of rewards and observations are identical:

\begin{assumption}\label{reward}
The reward function $R$ depends \textit{only} latent state and not on action.
\end{assumption}

Assumption \ref{reward} allows for the reward distribution $P^b(r_i|U_i)$ to also be written in terms of observable probabilities similar to $\pi^{(i)}_b(a_i|U_i)$, under analogous rank $|\mathcal{U}|$ assumptions. Additionally, in order to ensure uniqueness in the eigendecompositions, we further require certain sets of probabilities to be distinct. As these conditions are in terms of probability matrices which lack succinct representation, we relegate the assumptions and proof to Appendix \ref{is-identification-proof} but summarize here:

\begin{theorem}\label{IS-identification}
Under Assumptions \ref{necessary}--\ref{reward}, rank $|\mathcal{U}|$ and distinctness conditions, the vectors $\pi^{(i)}_b(a_i|U_i)$, and $P^b(r_i|U_i)$ are identifiable under the behavior measure w.r.t.~an unknown, but consistent, ordering on $U_i$ for all $i \leq H$.
\end{theorem}

Defining $\Pi_b(u_{0:t}) = \prod_{i=0}^t\pi^{(i)}_b(a_i|u_i)$ and letting $v$ denote value, we then obtain an IS procedure under the two following additional rank and positivity assumptions, as well as positivity and distinctness assumptions (see Appendix \ref{is-proof}):

\begin{assumption}\label{expo-rank-IS}
$\rank(P^b(R_i|U_i)) = |\mathcal{U}|$ 
\end{assumption}

\begin{assumption}\label{main-alg-pos}
$P^b(v|\tau^o) > 0, \forall v,\tau^o$ and $\pi_b^{(i)}(a_i|u_i) > 0, \forall i \leq H, a_i \in \mathcal{H}, u_i \in \mathcal{U}$.
\end{assumption}

\begin{theorem}\label{is-thm}
Under Assumptions \ref{necessary}--\ref{main-alg-pos}, as well as additional rank, distinctness, and positivity conditions, we have that
$v_H(\pi_e) =  \mathbb{E}[vW_{e,b}(v,\tau^o)]$ with \\
\centerline{$W_{e,b}(v,\tau^o) := \frac{1}{P^b(v|\tau^o)}\Pi_e(\tau^o)\Gamma_b(v,\tau^o)$}
and \\
\centerline{$\Gamma_b(v,\tau^o) = \sum_{r_{0:H}: \sum r_i = v}\sum_{u_{0:H}}\frac{\prod_{i=0}^HP^b(r_i|u_i)P^b(u_{0:H}|\tau^o)}{\Pi_b(u_{0:H})}$,}
where $W_{e,b}(v,\tau^o)$ are identifiable in terms of observable probabilities.
\end{theorem}

We give a sketch of the proof (for details, see Appendix \ref{is-proof}). 

\begin{proof}[Proof sketch]
We have that \begin{multline*}v_H(\pi_e) = \sum_{\tau \in \mathcal{T}_H}\sum_{v}vP^e(v|\tau)P^e(\tau) \\ = \sum_{\tau \in \mathcal{T}_H}\sum_v vP^b(v|u_{0:H})P^b(\tau)\left(\frac{\Pi_e(\tau^o)}{\Pi_b(u_{0:H})}\right)  \\ = \sum_{v}v\sum_{\tau^o \in \mathcal{T}^o_t}P^b(\tau^o)\Pi_e(\tau^o)\sum_{u_{0:H}}\frac{P^b(v|u_{0:H})P^b(u_{0:H}|\tau^o)}{\Pi_b(u_{0:H})} \\ = \sum_v vP^b(v|\tau^o) \sum_{\tau^o \in \mathcal{T}^o_t} W_{e,b}(v,\tau^o)P^b(\tau^o) = \mathbb{E}[vW_{e,b}(v,\tau^o)].\end{multline*} 

Identification of $W_{e,b}(v,\tau^o)$ is made possible by Theorem \ref{IS-identification} and writing $P^b(u_{0:H}|\tau^o)$ in terms observable and identifiable probabilities via eigenvalue analysis and pseudoinversion. 
\end{proof}

We also conducted experiments comparing our IS estimator to the baseline (biased) IS estimator of \citet{Tennenholtz_Shalit_Mannor_2020} and found that, in environments which meet the required conditions of Theorem \ref{IS-identification} and with sufficiently well-conditioned moment matrices that our estimator must pseudoinvert, it outperforms the baseline for horizon comprising two timesteps. However, for longer horizons our IS estimator often produces less accurate value estimates than the baseline, indicating that, while our IS estimator, in contrast to the baseline, is statistically consistent, the number of sampled trajectories required to guarantee good performance may often be quite high for longer horizons.

\section{Conclusion and Future Work}
In this paper we consider OPE in POMDPs. We extend prior work, relaxing assumptions on observable probability matrices to only have rank $|\mathcal{U}|$ rather than $|\mathcal{Z}|$, simultaneously generalizing to POMDPs in which $|\mathcal{Z}| \geq |\mathcal{U}|$. By additionally extending one of the proxy variables into the past, we further relax assumptions by assuming the matrices $P^b(Z_i,a_i, \mathcal{H}^o_{i-1})$ have rank $|\mathcal{U}|$, allowing for our estimator to apply to POMDPs in which the previous observation may not contain enough information about state, but the entire observable history does. We show that futures, however, cannot be immediately extended using this strategy, and thus make relatively weak assumptions on the reward distribution so as to allow for identification using rank $|\mathcal{U}|$ moment matrices involving both histories and futures. We then experimentally compare our main estimators and show that they both outperform the baseline given in \citet{Tennenholtz_Shalit_Mannor_2020} and perform well in a variety of different environments. Finally, we give a consistent IS algorithm for OPE in POMDPs which only depends on rank, distinctness, and positivity conditions on certain probability matrices and not on sufficiency assumptions.

\nocite{*}
\bibliographystyle{named.bst}
\bibliography{ijcai21.bib}

\newpage

\onecolumn
\appendix

\section{OPE via Observable Proxies}

We extend the identification scheme of \citet{Tennenholtz_Shalit_Mannor_2020} and \citet{miao2018identifying} to make weaker rank assumptions. The proofs in this section use the same general strategy as in \citet{Tennenholtz_Shalit_Mannor_2020}, but use spectral methods to relax their assumptions.

\subsection{One-Step Proxies}\label{one-step-proof}

\subsubsection{Proof of Lemma \ref{one-step-lemma}}
\begin{proof}
First, notice that $$P^b(U_{i+1}, z_i|a_i, U_i)P^b(U_i,a_i, Z_{i-1}) = P^b(U_{i+1}, z_i,a_i, Z_{i-1})$$ \begin{equation}\label{imp1}\implies P^b(U_{i+1}, z_i|a_i, U_i)P^b(U_i,a_i, Z_{i-1})M'_{i,a_i} = P^b(U_{i+1}, z_i,a_i, Z_{i-1})M'_{i,a_i}.\end{equation} Additionally, \begin{equation}\label{telescope}P^b(Z_i,a_i, Z_{i-1}) = P^b(Z_i|a_i, U_i)P^b(U_i,a_i, Z_{i-1})\end{equation} \begin{equation}\label{eq:one-step-rank_1}\implies P^b(Z_i,a_i, Z_{i-1})M'_{i,a_i} = P^b(Z_i|a_i, U_i)P^b(U_i,a_i, Z_{i-1})M'_{i,a_i}.\end{equation} From the definition of $M'_{i,a_i}$, we have that $$P^b(Z_i,a_i, Z_{i-1})M'_{i,a_i}$$ $$= P^b(Z_i,a_i, Z_{i-1})(M_{i,a_i}P^b(Z_i,a_i, Z_{i-1}))^+ = P^b(Z_i,a_i, Z_{i-1})V_{i,a_i}(I_{|\mathcal{U}|\times |\mathcal{Z}|} \Sigma_{i,a_i})^+$$ $$= U_{i,a_i}\Sigma_{i,a_i}(I_{|\mathcal{U}|\times |\mathcal{Z}|} \Sigma_{i,a_i})^+,$$ which has rank $|\mathcal{U}|$, thus implying, by equation \eqref{eq:one-step-rank_1}, that $\rank\left(P^b(U_i,a_i, Z_{i-1})M'_{i,a_i}\right) = |U|$, and so equation \eqref{imp1} implies that \begin{equation}\label{eq:first}P^b(U_{i+1}, z_i|a_i, U_i) = P^b(U_{i+1}, z_i,a_i, Z_{i-1})M'_{i,a_i}\left(P^b(U_i,a_i, Z_{i-1})M'_{i,a_i}\right)^{-1}.\end{equation} Now, notice that equation \eqref{telescope} implies that \begin{equation}\label{eq:rank-2}M_{i,a_i}P^b(Z_i|a_i, U_i)P^b(U_i,a_i, Z_{i-1}) = M_{i,a_i}P^b(Z_i,a_i, Z_{i-1}).\end{equation} Similarly notice that $\rank\left(M_{i,a_i}P^b(Z_i,a_i, Z_{i-1})\right) = |U|$ as, from the definition of $M_{i,a_i}$, we have $M_{i,a_i}P^b(Z_i,a_i, Z_{i-1}) = I_{|\mathcal{U}|\times |\mathcal{Z}|}\Sigma_{i,a_i}V^T_{i,a_i}$. This, in turn, implies that $\rank\left(M_{i,a_i}P^b(Z_i|a_i, U_i)\right) = |U|$ by equation \eqref{eq:rank-2}, and so equation \eqref{eq:rank-2} implies that $$P^b(U_i,a_i, Z_{i-1}) = \left(M_{i,a_i}P^b(Z_i|a_i, U_i)\right)^{-1}M_{i,a_i}P^b(Z_i,a_i, Z_{i-1}),$$ implying that \begin{equation}\label{second} P^b(U_i,a_i, Z_{i-1})M'_{i,a_i} = \left(M_{i,a_i}P^b(Z_i|a_i, U_i)\right)^{-1}M_{i,a_i}P^b(Z_i,a_i, Z_{i-1})M'_{i,a_i}.\end{equation} 

Now, we claim that $M_{i,a_i}P^b(Z_i,a_i, Z_{i-1})M'_{i,a_i} = I_{|\mathcal{U}|\times|\mathcal{U}|}$ since $$M_{i,a_i}P^b(Z_i,a_i, Z_{i-1})M'_{i,a_i} = I_{|\mathcal{U}|\times|\mathcal{Z}|}U^T_{i,a_i}U_{i,a_i}\Sigma_{i,a_i}V^T_{i,a_i}\left(I_{|\mathcal{U}|\times|\mathcal{Z}|}U^T_{i,a_i}U_{i,a_i}\Sigma_{i,a_i}V^T_{i,a_i}\right)^+$$ $$= I_{|\mathcal{U}|\times|\mathcal{Z}|}\Sigma_{i,a_i}(I_{|\mathcal{U}|\times|\mathcal{Z}|}\Sigma_{i,a_i})^+$$ and thus equation \eqref{second} becomes \begin{equation}\label{second'} P^b(U_i,a_i, Z_{i-1})M'_{i,a_i} = \left(M_{i,a_i}P^b(Z_i|a_i, U_i)\right)^{-1}.\end{equation} 

Finally, substituting \eqref{second'} into \eqref{eq:first} gives $$ P^b(U_{i+1}, z_i|a_i, U_i)$$ $$= P^b(U_{i+1}, z_i,a_i, Z_{i-1})M'_{i,a_i}M_{i,a_i}P^b(Z_i|a_i, U_i).$$ This, in combination with the fact that \begin{equation}\label{join}P^b(Z_i|a_i, U_i)P^b(U_{i}, z_{i-1},a_{i-1}, Z_{i-2}) = P^b(Z_i, z_{i-1},a_{i-1}, Z_{i-2})\end{equation} immediately implies that $$P^b(U_{i+1}, z_i|a_i, U_i)P^b(U_i, z_{i-1}|a_i, U_i) = P(U_{i+1}, z_i,a_i, Z_{i-1})M'_{i,a_i}$$ $$\cdot M_{i,a_i}P^b(Z_i, z_{i-1},a_{i-1}, Z_{i-2})M'_{a_{i-1}}M_{a_{i-1}}P^b(Z_{i-1}|a_{i-1}U_{i-1}),$$ as desired. Replacing $U_{i+1}$ with $r_i$ gives the second part of the Lemma.
\end{proof}

The proof of Theorem \ref{one-step-thm} then follows first from writing $P(r_t, z_t, \ldots, z_0||a_0, \ldots, a_t)$ as $$P^b(r_t,z_t|a_t, U_t)\left(\prod_{i=t-1}^0P^b(U_{i+1},z_i|a_i,U_i)\right)P^b(U_0),$$ using Lemma \ref{one-step-lemma} and equation \eqref{join} above to inductively show that \begin{multline}\label{observable-eq}P^b(r_t,z_t|a_t, U_t)\left(\prod_{i=t-1}^0P^b(U_{i+1},z_i|a_i,U_i)\right)P^b(U_0) \\ = P^b(r_t, z_t,a_t, Z_{t-1})M'_{t,a_t}\prod_{i=t-1}^0\left(M_{i+1,a_{i+1}}P^b(Z_{i+1}, z_i,a_i, Z_{i-1})M'_{i,a_i}\right)M_{0,a_0}P^b(Z_0),\end{multline} and then using the fact that $$P^e(r_t) = \sum_{\tau^o \in \mathcal{T}^o_t}\Pi_e(\tau^o)P(r_t, z_t, \ldots, z_0||a_0, \ldots, a_t)$$ to obtain the final result.

\subsection{Multi-Step Histories}\label{history-proof}

\subsubsection{Proof of Theorem \ref{history-thm}}\label{hist-proof}
The proof of Theorem \ref{history-thm} is essentially the same as that of Theorem \ref{one-step-thm}, except that we replace the one-step history $Z_{i-1}$ with $\mathcal{H}^o_{i-1}$. For completeness, we include it below, first proving a Lemma analogous to Lemma \ref{one-step-lemma}:

\begin{lemma}\label{history-lemma}
Given Assumption \ref{history-rank}, let $U_{i, a_i}\Sigma_{i, a_i}V^T_{i, a_i}$ be an SVD of $P^b(Z_i,a_i, \mathcal{H}^o_{i-1})$ with the diagonal entries of $\Sigma_{i,a_i}$ in descending order and define $M_{i, a_i} := I_{|\mathcal{U}|\times |\mathcal{Z}|}U_{i, a_i}^T$ and $M_{i, a_i}' := (M_{a_i}P^b(Z_i,a_i, \mathcal{H}^o_{i-1}))^+$. Then \begin{multline*}
    P^b(U_{i+1}, z_i|a_i, U_i)P^b(U_i, z_{i-1}|a_{i-1}, U_{i-1}) \\ = P^b(U_{i+1}, z_i,a_i, \mathcal{H}^o_{i-1})M'_{i,a_i}M_{i,a_i}P^b(Z_i, z_{i-1},a_{i-1}, \mathcal{H}^o_{i-2}) \\ \cdot M'_{i-1,a_{i-1}} M_{i-1,a_{i-1}}P^b(Z_{i-1}|a_{i-1}, U_{i-1})
\end{multline*} and \begin{multline*}
    P^b(r_{t}, z_t|a_t, U_t)P^b(U_t, z_{t-1}|a_{t-1}, U_{t-1}) \\ = P^b(r_t, z_t,a_t, \mathcal{H}^o_{t-1})M'_{t,a_t}M_{t,a_t}P^b(Z_t, z_{t-1},a_{t-1}, \mathcal{H}^o_{t-2}) \\ \cdot M'_{t-1,a_{t-1}} M_{t-1,a_{t-1}}P^b(Z_{t-1}|a_{t-1}, U_{t-1})
\end{multline*}
\end{lemma}

\begin{proof}
First, notice that $$P^b(U_{i+1}, z_i|a_i, U_i)P^b(U_i,a_i, \mathcal{H}^o_{i-1}) = P^b(U_{i+1}, z_i,a_i, \mathcal{H}^o_{i-1})$$ \begin{equation}\label{h-imp1}\implies P^b(U_{i+1}, z_i|a_i, U_i)P^b(U_i,a_i, \mathcal{H}^o_{i-1})M'_{i,a_i} = P^b(U_{i+1}, z_i,a_i, \mathcal{H}^o_{i-1})M'_{i,a_i}.\end{equation} Additionally, \begin{equation}\label{h-telescope}P^b(Z_i,a_i, \mathcal{H}^o_{i-1}) = P^b(Z_i|a_i, U_i)P^b(U_i,a_i, \mathcal{H}^o_{i-1})\end{equation} \begin{equation}\label{h-eq:one-step-rank_1}\implies P^b(Z_i,a_i, \mathcal{H}^o_{i-1})M'_{i,a_i} = P^b(Z_i|a_i, U_i)P^b(U_i,a_i, \mathcal{H}^o_{i-1})M'_{i,a_i}.\end{equation} From the definition of $M'_{i,a_i}$, we have that $$P^b(Z_i,a_i, \mathcal{H}^o_{i-1})M'_{i,a_i}$$ $$= P^b(Z_i,a_i, \mathcal{H}^o_{i-1})(M_{i,a_i}P^b(Z_i,a_i, \mathcal{H}^o_{i-1}))^+ = P^b(Z_i,a_i, \mathcal{H}^o_{i-1})V_{i,a_i}(I_{|\mathcal{U}|\times |\mathcal{Z}|} \Sigma_{i,a_i})^+$$ $$= U_{i,a_i}\Sigma_{i,a_i}(I_{|\mathcal{U}|\times |\mathcal{Z}|} \Sigma_{i,a_i})^+ = U_{i,a_i}I_{|\mathcal{Z}\times|\mathcal{U}|} = U_{i,a_i}I_{|\mathcal{Z}\times|\mathcal{U}|},$$ which has rank $|\mathcal{U}|$, thus implying, by equation \eqref{h-eq:one-step-rank_1}, that $\rank\left(P^b(U_i,a_i, \mathcal{H}^o_{i-1})M'_{i,a_i}\right) = |U|$, and so equation \eqref{h-imp1} implies that \begin{equation}\label{h-eq:first}P^b(U_{i+1}, z_i|a_i, U_i) = P^b(U_{i+1}, z_i,a_i, \mathcal{H}^o_{i-1})M'_{i,a_i}\left(P^b(U_i,a_i, \mathcal{H}^o_{i-1})M'_{i,a_i}\right)^{-1}.\end{equation} Now, notice that equation \eqref{h-telescope} implies that \begin{equation}\label{h-eq:rank-2}M_{i,a_i}P^b(Z_i|a_i, U_i)P^b(U_i,a_i, \mathcal{H}^o_{i-1}) = M_{i,a_i}P^b(Z_i,a_i, \mathcal{H}^o_{i-1}).\end{equation} Similarly notice that $\rank\left(M_{i,a_i}P^b(Z_i,a_i, \mathcal{H}^o_{i-1})\right) = |U|$ as, from the definition of $M_{i,a_i}$, we have $M_{i,a_i}P^b(Z_i,a_i, \mathcal{H}^o_{i-1}) = I_{|\mathcal{U}|\times |\mathcal{Z}|}\Sigma_{i,a_i}V^T_{i,a_i}$. This, in turn, implies that $\rank\left(M_{i,a_i}P^b(Z_i|a_i, U_i)\right) = |U|$ by equation \eqref{h-eq:rank-2}, and so equation \eqref{h-eq:rank-2} implies that $$P^b(U_i,a_i, \mathcal{H}^o_{i-1}) = \left(M_{i,a_i}P^b(Z_i|a_i, U_i)\right)^{-1}M_{i,a_i}P^b(Z_i,a_i, \mathcal{H}^o_{i-1}),$$ implying that \begin{equation}\label{h-second} P^b(U_i,a_i, \mathcal{H}^o_{i-1})M'_{i,a_i} = \left(M_{i,a_i}P^b(Z_i|a_i, U_i)\right)^{-1}M_{i,a_i}P^b(Z_i,a_i, \mathcal{H}^o_{i-1})M'_{i,a_i}.\end{equation} 

Now, we claim that $M_{i,a_i}P^b(Z_i,a_i, \mathcal{H}^o_{i-1})M'_{i,a_i} = I_{|\mathcal{U}|\times|\mathcal{U}|}$ since $$M_{i,a_i}P^b(Z_i,a_i, \mathcal{H}^o_{i-1})M'_{i,a_i} = I_{|\mathcal{U}|\times|\mathcal{Z}|}U^T_{i,a_i}U_{i,a_i}\Sigma_{i,a_i}V^T_{i,a_i}\left(I_{|\mathcal{U}|\times|\mathcal{Z}|}U^T_{i,a_i}U_{i,a_i}\Sigma_{i,a_i}V^T_{i,a_i}\right)^+$$ $$ = I_{|\mathcal{U}|\times|\mathcal{Z}|}\Sigma_{i,a_i}(I_{|\mathcal{U}|\times|\mathcal{Z}|}\Sigma_{i,a_i})^+$$ and thus equation \eqref{h-second} becomes \begin{equation}\label{h-second'} P^b(U_i,a_i, \mathcal{H}^o_{i-1})M'_{i,a_i} = \left(M_{i,a_i}P^b(Z_i|a_i, U_i)\right)^{-1}.\end{equation} 

Finally, substituting \eqref{h-second'} into \eqref{h-eq:first} gives $$ P^b(U_{i+1}, z_i|a_i, U_i)$$ $$= P^b(U_{i+1}, z_i,a_i, \mathcal{H}^o_{i-1})M'_{i,a_i}M_{i,a_i}P^b(Z_i|a_i, U_i).$$ This, in combination with the fact that \begin{equation}\label{h-join}P^b(Z_i|a_i, U_i)P^b(U_{i}, z_{i-1},a_{i-1}, \mathcal{H}^o_{i-2}) = P^b(Z_i, z_{i-1},a_{i-1}, \mathcal{H}^o_{i-2})\end{equation} immediately implies that $$P^b(U_{i+1}, z_i|a_i, U_i)P^b(U_i, z_{i-1}|a_i, U_i) = P(U_{i+1}, z_i,a_i, \mathcal{H}^o_{i-1})M'_{i,a_i}$$ $$\cdot M_{i,a_i}P^b(Z_i, z_{i-1},a_{i-1}, \mathcal{H}^o_{i-2})M'_{a_{i-1}}M_{a_{i-1}}P^b(Z_{i-1}|a_{i-1}U_{i-1}),$$ as desired. Replacing $U_{i+1}$ with $r_i$ and $i$ with gives the second part of the Lemma.
\end{proof}

Again, the proof of Theorem \ref{history-thm} follows first from writing $P(r_t, z_t, \ldots, z_0||a_0, \ldots, a_t)$ as $$P^b(r_t,z_t|a_t, U_t)\left(\prod_{i=t-1}^0P^b(U_{i+1},z_i|a_i,U_i)\right)P^b(U_0),$$ using Lemma \ref{history-lemma} and equation \eqref{h-join} above to inductively show that \begin{multline}\label{h-observable-eq}P^b(r_t,z_t|a_t, U_t)\left(\prod_{i=t-1}^0P^b(U_{i+1},z_i|a_i,U_i)\right)P^b(U_0) \\ = P^b(r_t, z_t,a_t, \mathcal{H}^o_{t-1})M'_{t,a_t} \prod_{i=t-1}^0\left(M_{i+1,a_{i+1}}P^b(Z_{i+1}, z_i,a_i, \mathcal{H}^o_{i-1})M'_{i,a_i}\right)M_{0,a_0}P^b(Z_0),\end{multline} and again using the fact that $$P^e(r_t) = \sum_{\tau^o \in \mathcal{T}^o_t}\Pi_e(\tau^o)P(r_t, z_t, \ldots, z_0||a_0, \ldots, a_t)$$ to obtain the final result.

\subsection{Multi-Step Futures}\label{future-proof}

\subsubsection{Proof of equation \eqref{future-eq} of Theorem \ref{future_thm}}
The proof of equation \eqref{future-eq} of Theorem \ref{future_thm} is, again, essentially the same as those of Theorems \ref{one-step-thm} and \ref{history-thm}, except that we replace the one-step history $Z_{i-1}$ with $\mathcal{H}^o_{i-1}$ and the one-step future $Z_i$ with $\mathcal{F}^o_i$ and here, we use moment matrices of conditional probabilities rather than joints. Again, for completeness, we include it below, first proving a Lemma analogous to Lemmas \ref{one-step-lemma} and \ref{history-lemma}:

\begin{lemma}\label{future-lemma}
Given Assumption \ref{future-rank}, let $U_{i, a_i}\Sigma_{i, a_i}V^T_{i, a_i}$ be an SVD of $P^b(\mathcal{F}^o_i|a_i, \mathcal{H}^o_{i-1})$ with the diagonal entries of $\Sigma_{i,a_i}$ in descending order and define $M_{i, a_i} := I_{|\mathcal{U}|\times |\mathcal{Z}|}U_{i, a_i}^T$ and $M_{i, a_i}' := (M_{a_i}P^b(\mathcal{F}^o_i|a_i, \mathcal{H}^o_{i-1}))^+$. Then \begin{multline*}
    P^b(U_{i+1}, z_i|a_i, U_i)P^b(U_i, z_{i-1}|a_{i-1}, U_{i-1}) \\ = P^b(U_{i+1}, z_i|a_i, \mathcal{H}^o_{i-1})M'_{i,a_i}M_{i,a_i}P^b(\mathcal{F}^o_i, z_{i-1}|a_{i-1}, \mathcal{H}^o_{i-2}||a_i) \\ \cdot M'_{i-1,a_{i-1}} M_{i-1,a_{i-1}}P^b(\mathcal{F}^o_{i-1}|a_{i-1}, U_{i-1})
\end{multline*} and \begin{multline*}
    P^b(r_{t}, z_t|a_t, U_t)P^b(U_t, z_{t-1}|a_{t-1}, U_{t-1}) \\ = P^b(r_t, z_t|a_t, \mathcal{H}^o_{t-1})M'_{t,a_t}M_{t,a_t}P^b(\mathcal{F}^o_t, z_{t-1}|a_{t-1}, \mathcal{H}^o_{t-2}||a_i) \\ \cdot M'_{t-1,a_{t-1}} M_{t-1,a_{t-1}}P^b(\mathcal{F}^o_{t-1}|a_{t-1}, U_{t-1})
\end{multline*}
\end{lemma}

\begin{proof}
First, notice that $$P^b(U_{i+1}, z_i|a_i, U_i)P^b(U_i|a_i, \mathcal{H}^o_{i-1}) = P^b(U_{i+1}, z_i|a_i, \mathcal{H}^o_{i-1})$$ \begin{equation}\label{f-imp1}\implies P^b(U_{i+1}, z_i|a_i, U_i)P^b(U_i|a_i, \mathcal{H}^o_{i-1})M'_{i,a_i} = P^b(U_{i+1}, z_i|a_i, \mathcal{H}^o_{i-1})M'_{i,a_i}.\end{equation} Additionally, \begin{equation}\label{f-telescope}P^b(\mathcal{F}^o_i|a_i, \mathcal{H}^o_{i-1}) = P^b(\mathcal{F}^o_i|a_i, U_i)P^b(U_i|a_i, \mathcal{H}^o_{i-1})\end{equation} \begin{equation}\label{f-eq:one-step-rank_1}\implies P^b(\mathcal{F}^o_i|a_i, \mathcal{H}^o_{i-1})M'_{i,a_i} = P^b(\mathcal{F}^o_i|a_i, U_i)P^b(U_i|a_i, \mathcal{H}^o_{i-1})M'_{i,a_i}.\end{equation} From the definition of $M'_{i,a_i}$, we have that $$P^b(\mathcal{F}^o_i|a_i, \mathcal{H}^o_{i-1})M'_{i,a_i}$$ $$= P^b(\mathcal{F}^o_i|a_i, \mathcal{H}^o_{i-1})(M_{i,a_i}P^b(\mathcal{F}^o_i|a_i, \mathcal{H}^o_{i-1}))^+ = P^b(\mathcal{F}^o_i|a_i, \mathcal{H}^o_{i-1})V_{i,a_i}(I_{|\mathcal{U}|\times |\mathcal{Z}|} \Sigma_{i,a_i})^+$$ $$= U_{i,a_i}\Sigma_{i,a_i}(I_{|\mathcal{U}|\times |\mathcal{Z}|} \Sigma_{i,a_i})^+ = U_{i,a_i}I_{|\mathcal{Z}\times|\mathcal{U}|} = U_{i,a_i}I_{|\mathcal{Z}\times|\mathcal{U}|},$$ which has rank $|\mathcal{U}|$, thus implying, by equation \eqref{f-eq:one-step-rank_1}, that $\rank\left(P^b(U_i|a_i, \mathcal{H}^o_{i-1})M'_{i,a_i}\right) = |U|$, and so equation \eqref{f-imp1} implies that \begin{equation}\label{f-eq:first}P^b(U_{i+1}, z_i|a_i, U_i) = P^b(U_{i+1}, z_i|a_i, \mathcal{H}^o_{i-1})M'_{i,a_i}\left(P^b(U_i|a_i, \mathcal{H}^o_{i-1})M'_{i,a_i}\right)^{-1}.\end{equation} Now, notice that equation \eqref{f-telescope} implies that \begin{equation}\label{f-eq:rank-2}M_{i,a_i}P^b(\mathcal{F}^o_i|a_i, U_i)P^b(U_i|a_i, \mathcal{H}^o_{i-1}) = M_{i,a_i}P^b(\mathcal{F}^o_i|a_i, \mathcal{H}^o_{i-1}).\end{equation} Similarly notice that $\rank\left(M_{i,a_i}P^b(\mathcal{F}^o_i|a_i, \mathcal{H}^o_{i-1})\right) = |U|$ as, from the definition of $M_{i,a_i}$, we have $M_{i,a_i}P^b(\mathcal{F}^o_i|a_i, \mathcal{H}^o_{i-1}) = I_{|\mathcal{U}|\times |\mathcal{Z}|}\Sigma_{i,a_i}V^T_{i,a_i}$. This, in turn, implies that $\rank\left(M_{i,a_i}P^b(\mathcal{F}^o_i|a_i, U_i)\right) = |U|$ by equation \eqref{f-eq:rank-2}, and so equation \eqref{f-eq:rank-2} implies that $$P^b(U_i|a_i, \mathcal{H}^o_{i-1}) = \left(M_{i,a_i}P^b(\mathcal{F}^o_i|a_i, U_i)\right)^{-1}M_{i,a_i}P^b(\mathcal{F}^o_i|a_i, \mathcal{H}^o_{i-1}),$$ implying that \begin{equation}\label{f-second} P^b(U_i|a_i, \mathcal{H}^o_{i-1})M'_{i,a_i} = \left(M_{i,a_i}P^b(\mathcal{F}^o_i|a_i, U_i)\right)^{-1}M_{i,a_i}P^b(\mathcal{F}^o_i|a_i, \mathcal{H}^o_{i-1})M'_{i,a_i}.\end{equation} 

Now, we claim that $M_{i,a_i}P^b(\mathcal{F}^o_i|a_i, \mathcal{H}^o_{i-1})M'_{i,a_i} = I_{|\mathcal{U}|\times|\mathcal{U}|}$ since $$M_{i,a_i}P^b(\mathcal{F}^o_i|a_i, \mathcal{H}^o_{i-1})M'_{i,a_i} = I_{|\mathcal{U}|\times|\mathcal{Z}|}U^T_{i,a_i}U_{i,a_i}\Sigma_{i,a_i}V^T_{i,a_i}\left(I_{|\mathcal{U}|\times|\mathcal{Z}|}U^T_{i,a_i}U_{i,a_i}\Sigma_{i,a_i}V^T_{i,a_i}\right)^+$$ $$ = I_{|\mathcal{U}|\times|\mathcal{Z}|}\Sigma_{i,a_i}(I_{|\mathcal{U}|\times|\mathcal{Z}|}\Sigma_{i,a_i})^+$$ and thus equation \eqref{f-second} becomes \begin{equation}\label{f-second'} P^b(U_i|a_i, \mathcal{H}^o_{i-1})M'_{i,a_i} = \left(M_{i,a_i}P^b(\mathcal{F}^o_i|a_i, U_i)\right)^{-1}.\end{equation} 

Finally, substituting \eqref{f-second'} into \eqref{f-eq:first} gives $$ P^b(U_{i+1}, z_i|a_i, U_i)$$ $$= P^b(U_{i+1}, z_i|a_i, \mathcal{H}^o_{i-1})M'_{i,a_i}M_{i,a_i}P^b(\mathcal{F}^o_i|a_i, U_i).$$ This, in combination with the fact that \begin{equation}\label{f-join}P^b(\mathcal{F}^o_i|a_i, U_i)P^b(U_{i}, z_{i-1}|a_{i-1}, \mathcal{H}^o_{i-2}) = P^b(\mathcal{F}^o_i, z_{i-1}|a_{i-1}, \mathcal{H}^o_{i-2}||a_i)\end{equation} immediately implies that $$P^b(U_{i+1}, z_i|a_i, U_i)P^b(U_i, z_{i-1}|a_i, U_i) = P(U_{i+1}, z_i|a_i, \mathcal{H}^o_{i-1})M'_{i,a_i}$$ $$\cdot M_{i,a_i}P^b(\mathcal{F}^o_i, z_{i-1}|a_{i-1}, \mathcal{H}^o_{i-2}||a_i)M'_{a_{i-1}}M_{a_{i-1}}P^b(Z_{i-1}|a_{i-1}U_{i-1}),$$ as desired. Replacing $U_{i+1}$ with $r_i$ and $i$ with gives the second part of the Lemma.
\end{proof}

Again, the proof of Theorem \ref{future_thm} follows first from writing $P(r_t, z_t, \ldots, z_0||a_0, \ldots, a_t)$ as $$P^b(r_t,z_t|a_t, U_t)\left(\prod_{i=t-1}^0P^b(U_{i+1},z_i|a_i,U_i)\right)P^b(U_0),$$ using Lemma \ref{future-lemma} and equation \eqref{f-join} above to inductively show that \begin{multline}\label{f-observable-eq}P^b(r_t,z_t|a_t, U_t)\left(\prod_{i=t-1}^0P^b(U_{i+1},z_i|a_i,U_i)\right)P^b(U_0) \\ = P^b(r_t, z_t|a_t, \mathcal{H}^o_{t-1})M'_{t,a_t} \prod_{i=t-1}^0\left(M_{i+1,a_{i+1}}P^b(\mathcal{F}^o_{i+1}, z_i|a_i, \mathcal{H}^o_{i-1}||a_{i+1})M'_{i,a_i}\right)M_{0,a_0}P^b(Z_0),\end{multline} and again using the fact that $$P^e(r_t) = \sum_{\tau^o \in \mathcal{T}^o_t}\Pi_e(\tau^o)P(r_t, z_t, \ldots, z_0||a_0, \ldots, a_t)$$ to obtain equation \eqref{future-eq}.

\subsubsection{Identification of $P^b(\mathcal{F}^o_{i+1}, z_i|a_i, \mathcal{H}^o_{i-1}||a_{i+1})$}
Our identification strategy for $P^b(\mathcal{F}^o_{i+1}, z_i|a_i, \mathcal{H}^o_{i-1}||a_{i+1})$ in Theorem \ref{future_thm} extends that of \citet{kuroki}. Define the following matrices $$P^{\tau}_{i,a_i} := \begin{pmatrix}
1 & P^b(\mathcal{F}_i^{o^1}|a_i) & \cdots & P^b(\mathcal{F}_i^{o^{|\mathcal{F}^o_i|-1}}|a_i)\\
P^b(\mathcal{H}_{i-1}^{o^1}|a_i) & P^b(\mathcal{F}_i^{o^1}, \mathcal{H}_{i-1}^{o^1}|a_i) & \cdots & P^b(\mathcal{F}_i^{o^{|\mathcal{F}_i|-1}}, \mathcal{H}_{i-1}^{o^1}|a_i)\\
\vdots & \vdots & \ddots & \vdots\\
P^b(\mathcal{H}_{i-1}^{o^{|\mathcal{H}_{i-1}|-1}}|a_i) & P^b(\mathcal{F}_{i}^{o^1}, \mathcal{H}_{i-1}^{o^{|\mathcal{H}^o_{i-1}|-1}}|a_i) & \cdots & P^b(\mathcal{F}_i^{o^{|\mathcal{F}^o_i|-1}}, \mathcal{H}_{i-1}^{o^{|\mathcal{H}^o_{i-1}|-1}}|a_i)
\end{pmatrix},$$  $$Q^{\tau}_{i, a_i, r_{i}} := \begin{pmatrix}
P^b(r_i|a_i) & P^b(\mathcal{F}_i^{o^1}, r_i|a_i) & \cdots & P^b(\mathcal{F}_i^{o^{|\mathcal{F}_i|-1}},r_i|a_i)\\
P^b(\mathcal{H}_{i-1}^{o^1}, r_i|a_i) & P^b(\mathcal{F}_i^{o^1}, \mathcal{H}_{i-1}^{o^1}, r_i|a_i) & \cdots & P^b(\mathcal{F}_i^{o^{|\mathcal{F}^o_i|-1}}, \mathcal{H}_{i-1}^{o^1}, r_i|a_i)\\
\vdots & \vdots & \ddots & \vdots\\
P^b(\mathcal{H}_{i-1}^{o^{|\mathcal{H}^o_{i-1}|-1}}, r_i|a_i) & P^b(\mathcal{F}_{i}^{o^1}, \mathcal{H}_{i-1}^{o^{|\mathcal{H}^o_{i-1}|-1}}, r_i|a_i) & \cdots & P^b(\mathcal{F}_i^{o^{|\mathcal{F}^o_i|-1}}, \mathcal{H}_{i-1}^{o^{|\mathcal{H}^o_{i-1}|-1}}, r_i|a_i)
\end{pmatrix},$$ $$U^{\tau}_{i, a_i} := \begin{pmatrix}
1 & P^b(\mathcal{F}_i^{o^1}|u_i^{(1)}, a_i) & \cdots & P^b(\mathcal{F}_i^{o^{|\mathcal{F}_i|-1}}|u_i^{(1)}, a_i)\\
\vdots & \vdots & \ddots & \vdots\\
1 & P^b(\mathcal{F}_i^{o^1}|u_i^{(\kappa)}, a_i) & \cdots & P^b(\mathcal{F}_i^{o^{|\mathcal{F}_i|-1}}|u_i^{(\kappa)}, a_i)
\end{pmatrix},$$ $$R^\tau_{i, a_i} = \begin{pmatrix}
1 & P^b(\mathcal{H}_{i-1}^{o^1}|a_i, u_i^{(1)}) & \cdots & P^b(\mathcal{H}_{i-1}^{o^{|\mathcal{H}_{i-1}|-1}}|a_i, u_i^{(1)})\\
\vdots & \vdots & \ddots & \vdots\\
1 & P^b(\mathcal{H}_{i-1}^{o^1}|a_i, u_i^{(\kappa)}) & \cdots & P^b(\mathcal{H}_{i-1}^{o^{|\mathcal{H}_{i-1}|-1}}|a_i, u_i^{(\kappa)})
\end{pmatrix},$$ $$M^\tau_{i,a_i} := \diag(P^b(u_i^{(1)}|a_i), \ldots, P^b(u_i^{(\kappa)}|a_i)).$$ Lastly, define $$\Delta_{i,a_{i}, r_i} := \diag(P^b(r_i|a_i, u_i^{(1)}), \ldots, P^b(r_{i}|a_i, u_i^{(\kappa)}).$$ The ordering $u_i^{(j)}$ is such that the elements on the diagonal of $\Delta_{i,z_{i+1}}$ are in non-decreasing order. That is, $P^b(r_i|a_i, u_i^{(1)}) \geq \ldots \geq P^b(r_{i}|a_i, u_i^{(\kappa)})$, where $\kappa = |\mathcal{U}|$. Furthermore, $S^j$ denotes the $j$th element of set $S$.

Now, we claim that Assumptions \ref{future-rank} and \ref{reward-pos-distinct} imply the following
\begin{lemma}\label{appendix-future-rank}
\begin{equation}\label{ranks}\rank(P^\tau_{i,a_i}) = \rank(Q^\tau_{i,a_i,r_i}) = |\mathcal{U}|\end{equation}
\end{lemma}
\begin{proof}
Define $$E_d = \begin{pmatrix}
0 & 0 & \cdots & 0 & 1\\
1 & 0 & \cdots & 0 & -1\\
0 & 1 & \cdots & 0 & -1\\
\vdots & \vdots & \ddots & \vdots & \vdots\\
0 & 0 & \cdots & 1 & -1
\end{pmatrix} \in \mathbb{R}^{d\times d}$$ It is clear that $$E^T_{|\mathcal{H}_{i-1}|}P^{\tau}_{i,a_i}E_{|\mathcal{F}_i|} = P^b(\mathcal{H}_{i-1}, \mathcal{F}_i|a_i) \text{ and } E^T_{|\mathcal{H}_{i-1}|}Q^{\tau}_{i,a_i, r_i}E_{|\mathcal{F}_i|} = P^b(\mathcal{H}_{i-1}, \mathcal{F}_i, r_i|a_i).$$ Hence, it suffices to show that $$\rank(P^b(\mathcal{H}_{i-1}, \mathcal{F}_i|a_i)) = \rank(P^b(\mathcal{H}_{i-1}, \mathcal{F}_i, r_i|a_i)) = |\mathcal{U}|,$$ as $E_d$ is clearly invertible for all $d$.

To see why $$\rank(P^b(\mathcal{H}_{i-1}, \mathcal{F}_i|a_i)) = \rank(P^b(\mathcal{H}_{i-1}, \mathcal{F}_i, r_i|a_i)) = |\mathcal{U}|,$$ note that we have $$P^b(\mathcal{H}_{i-1}, \mathcal{F}_i|a_i)^T = P^b(\mathcal{F}_i, \mathcal{H}_{i-1}|a_i)$$ $$= P^b(\mathcal{F}_i|a_i, U_i)P^b(U_i, \mathcal{H}_{i-1}|a_i)$$ $$ = P^b(\mathcal{F}_i|a_i, U_i)P^b(U_i|a_i, \mathcal{H}_{i-1})\diag(P^b(\mathcal{H}_{i-1}|a_i)) = P^b(\mathcal{F}_i|a_i, \mathcal{H}_{i-1})\diag(P^b(\mathcal{H}_{i-1}|a_i)).$$ We are guaranteed that $\diag(P^b(\mathcal{H}_{i-1}|a_i))$ is invertible (otherwise we would be conditioning on measure $0$ events), and thus invertibility of $P^{\tau}_{i,a_i}$ follows from Assumption \ref{future-rank}. As for $Q^{\tau}_{i,a_i, r_i}$, note that $$P^b(\mathcal{H}_{i-1}, \mathcal{F}_i, r_i|a_i)^T =  P^b(\mathcal{F}_i|a_i, U_i)\diag(P^b(r_i|a_i, U_i))P^b(U_i|a_i, \mathcal{H}_{i-1})\diag(P^b(\mathcal{H}_{i-1}|a_i)).$$ As $\diag(P^b(r_i|a_i, U_i))$ is invertible (the positivity of Assumption \ref{reward-pos-distinct}), we have that $Q^{\tau}_{i,a_i, r_i}$ will have rank $|\mathcal{U}|$.
\end{proof}

Now, notice that \begin{equation}\label{fPa}
    P^\tau_{i,a_i} = (R^\tau_{i,a_i}{})^TM^\tau_{i,a_i}U^\tau_{i, a_i}\end{equation} and 
    \begin{equation}\label{fQa}Q^\tau_{i,a_i,r_i} = (R^\tau_{i,a_i}{})^TM^\tau_{i,a_i}\Delta_{i,a_{i}, r_i}U^\tau_{i, a_i}
\end{equation}

 Now, for any subset $\Lambda$ of $\{2, \ldots, |\mathcal{F}^o_i|\}$ of size $|\mathcal{U}|-1$, define $U^\tau_{i,a_i}(\Lambda)$ to be the $|\mathcal{U}|\times|\mathcal{U}|$ matrix whose first column is the first column of $U^\tau_{i,a_i}$ and whose $j$th column is the $\Lambda_{j-1}$th column of $U^\tau_{i,a_i}$ for all $2 \leq j \leq |\mathcal{U}|$, where $\Lambda_j$ is the $j$th smallest element of $\Lambda$. Equivalently, $U^\tau_{i,a_i}(\Lambda)$ can be written as $U^\tau_{i,a_i}L(\Lambda)$ where $L(\Lambda)$ is the $|\mathcal{F}^o_i|\times|\mathcal{U}|$ matrix whose first column is $e_1$ and whose $j$th column is $e_{\Lambda_{j-1}}$, where $e_k$ denotes the $k$th standard basis vector of $\mathbb{R}^{|\mathcal{F}^o_i|}$, for $2 \leq j \leq |\mathcal{U}|$.
\\
\\
Note that we can identify a maximal set of $|\mathcal{U}|$ linearly independent columns of $P^\tau_{i,a_i}$, containing the first column, via Gaussian elimination. Let $\Lambda$ be the set of indices of these columns. Right-multiplying equations \eqref{fPa} and \eqref{fQa} by $L(\Lambda)$ thus gives, \begin{equation}\label{fPa'}
    P^\tau_{i,a_i}L(\Lambda) = (R^\tau_{i,a_i}{})^TM^\tau_{i,a_i}U^\tau_{i, a_i}(\Lambda)\end{equation} and 
    \begin{equation}\label{fQa'}Q^\tau_{i,a_i,r_i}L(\Lambda) = (R^\tau_{i,a_i}{})^TM^\tau_{i,a_i}\Delta_{i,a_{i}, r_i}U^\tau_{i, a_i}(\Lambda)
\end{equation}  By assumption,  $U^\tau_{i, a_i}(\Lambda)$ is invertible. Additionally, Lemma \ref{appendix-future-rank} implies that $M^\tau_{i,a_i}$ is invertible and $\rank(R^\tau_{i,a_i}) = |\mathcal{U}|$, via equation \eqref{fPa}, and so, from equation \eqref{fPa'}, we have $(P^\tau_{i,a_i}L(\Lambda))^+ = (U^\tau_{i, a_i}(\Lambda))^{-1}(M^\tau_{i,a_i})^{-1}((R^\tau_{i,a_i}{})^T)^+$. This along with equation \eqref{fQa'} gives the following:

\begin{lemma}\label{fdiagonalization}
Under Assumption \ref{appendix-future-rank}, we have that $$(P^\tau_{i,a_i}L(\Lambda))^+Q^\tau_{i,a_i,r_i}L(\Lambda) = (U^\tau_{i, a_i}(\Lambda))^{-1}\Delta_{i,a_{i}, r_i}U^\tau_{i, a_i}(\Lambda).$$
\end{lemma}

Lemma \ref{fdiagonalization} then allows us to diagonalize the \textit{observable} product $(P^\tau_{i,a_i}L(\Lambda))^+Q^\tau_{i,a_i,r_i}L(\Lambda)$, as in \citet{kuroki}, to obtain the left eigenvalue matrix $T^\tau_{i,a_i}(\Lambda)$. Note that each row of $T^\tau_{i,a_i}(\Lambda)$ differs from the corresponding row of $U^\tau_{i, a_i}(\Lambda)$ only by a constant factor and that the first entry of each row of $U^\tau_{i, a_i}(\Lambda)$ is $1$. This is guaranteed by the distinctness of Assumption \ref{reward-pos-distinct}, which guarantees distinct eigenvalues, $\Delta_{i,a_i, r_i}$, and thus a unique ordering on the eigenvectors of $T^\tau_{i,a_i}(\Lambda)$. Hence, we have that $U^\tau_{i, a_i}(\Lambda) = \diag((T^\tau_{i,a_i})_1)^{-1}T^\tau_{i,a_i}(\Lambda),$ where $(T^\tau_{i,a_i})_1$ denotes the first column of $T^\tau_{i,a_i}$, thereby allowing for the successful identification of $U^\tau_{i, a_i}(\Lambda)$.

Finally, the remaining columns of $U^\tau_{i,a_i}$ outside of those indexed in $\Lambda$ are identifiable by simply determining the linear dependence among the columns of $P^\tau_{i,a_i}$ and recognizing that they are the same for $U^\tau_{i,a_i}$. That is, if, for $j \not\in \{1\}\cup\Lambda$ $$(P^\tau_{i,a_i})_j = \sum_{j' \in \{1\}\cup\Lambda}c^j_{j'}(P^\tau_{i,a_i})_{j'},$$ for constants $c_{j'}$, then $$(U^\tau_{i,a_i})_j = \sum_{j' \in \{1\}\cup\Lambda}c^j_{j'}(U^\tau_{i,a_i})_{j'},$$ thus allowing for the identification of all columns of $U^\tau_{i,a_i}$. Again, the coefficients $c^j_{j'}$, for each $j \not\in \{1\}\cup\Lambda, j' \in \{1\}\cup\Lambda$ are identifiable via Gaussian elimination on $P^\tau_{i,a_i}$.

\begin{remark}\label{noisy-remark}
We remark that, in practice, our estimates of the matrices $P^\tau_{i,a_i}$ will be noisy and thus often have rank greater than $|\mathcal{U}|$ (even if the true moment matrix has rank $|\mathcal{U}|$). In such cases, instead of identifying $|\mathcal{U}|$ linearly independent columns which also contain the first, one may consider all matrices $P^\tau_{i,a_i}L(\Lambda)$ for all choices of $\Lambda \subseteq \{2, \ldots, |\mathcal{F}^o_i|\}$ of size $|\mathcal{U}|-1$ and select the choice of $\Lambda$ corresponding to the matrix of lowest condition number. Alternatively, this can be mitigated by making the support of $f^o_i$ more coarse---by combining elements of the support---so that $P^\tau_{i,a_i}, Q^\tau_{i,a_i, r_i}$ and $U^\tau_{i,a_i}$ all have second dimension equal to $|\mathcal{U}|$ instead of $|\mathcal{F}^o_i|$. 
\end{remark}

The last remaining step is to use the identification of $U^\tau_{i+1,a_{i+1}}$ to identify $$P^b(\mathcal{F}^o_{i+1}, z_i|a_i, \mathcal{H}^o_{i-1}||a_{i+1}) = P^b(\mathcal{F}_{i+1}|a_{i+1}, U_{i+1})P^b(U_{i+1}, z_{i}|a_{i}, \mathcal{H}^o_{i-1}).$$ Note that $$P^b(\mathcal{F}_{i+1}|a_{i+1}, U_{i+1})P^b(U_{i+1}, a_{i+1}, z_{i}|a_{i}, \mathcal{H}^o_{i-1}) = P^b(\mathcal{F}_{i+1}, a_{i+1},z_{i}| a_{i}, \mathcal{H}^o_{i-1}).$$ The fact that $P^b(\mathcal{F}_{i+1}|a_{i+1}, U_{i+1})$ has full column rank and is identifiable then immediately implies the identifiability of $P^b(U_{i+1}, a_{i+1}, z_{i}|a_{i}, \mathcal{H}^o_{i-1})$ as $$P^b(U_{i+1}, a_{i+1}, z_{i}|a_{i}, \mathcal{H}^o_{i-1}) = (P^b(\mathcal{F}_{i+1}|a_{i+1}, U_{i+1}))^+P^b(\mathcal{F}_{i+1}, a_{i+1},z_{i}| a_{i}, \mathcal{H}^o_{i-1}),$$ and because $P^b(\mathcal{F}_{i+1}, a_{i+1},z_{i}| a_{i}, \mathcal{H}^o_{i-1})$ is observable. Then, $P^b(U_{i+1}, z_{i}|a_{i}, \mathcal{H}^o_{i-1})$ is readily obtained via marginalization across $a_{i+1}$ in $P^b(U_{i+1}, a_{i+1}, z_{i}|a_{i}, \mathcal{H}^o_{i-1})$. Finally, as we have identified both $P^b(\mathcal{F}_{i+1}|a_{i+1}, U_{i+1})$ and $P^b(U_{i+1}, z_{i}|a_{i}, \mathcal{H}^o_{i-1})$, the identification of $P^b(\mathcal{F}^o_{i+1}, z_i|a_i, \mathcal{H}^o_{i-1}||a_{i+1})$ follows from the fact that $$P^b(\mathcal{F}^o_{i+1}, z_i|a_i, \mathcal{H}^o_{i-1}||a_{i+1}) = P^b(\mathcal{F}_{i+1}|a_{i+1}, U_{i+1})P^b(U_{i+1}, z_{i}|a_{i}, \mathcal{H}^o_{i-1}).$$




\section{Experimental Details}\label{exper-details}
\subsection{Choice of Parameters in Experiment 2}
In these experiments, we give a choice of parameters violating Assumption \ref{invertibility} by making $\rank(P^b(U_i|a_i, Z_{i-1})) = 1 < |\mathcal{U}|$.

Define $$P^b(Z_i|U_i) = \begin{pmatrix}
\rho_0 & \rho_1\\
1 - \rho_0 & 1 - \rho_1
\end{pmatrix}, P^b(U_{i+1}|a_i=0, U_i) = \begin{pmatrix}
a & c\\
1 - a & 1 - c
\end{pmatrix}, P^b(U_{i+1}|a_i=1, U_i) = \begin{pmatrix}
b & d\\
1 - b & 1 - d
\end{pmatrix},$$ $$\pi_b^{(i)}(A_i|U_i) = \begin{pmatrix}
\epsilon & \delta\\
1-\epsilon & 1-\delta
\end{pmatrix}, \text{ and }P^b(U_i) = \begin{pmatrix}
s_i\\
1-s_i
\end{pmatrix}.$$ Note that $s_i$ for $i > 0$ are not actually free to choose (but rather determined by the other parameters). Also, define $$P^b(Z_{i-1}|U_i) = \begin{pmatrix}
\zeta^i_0 & \zeta^i_1\\
1 - \zeta^i_0 & 1 - \zeta^i_1
\end{pmatrix}.$$ Noting that $$P^b(u_i|a_i, z_{i-1}) = \frac{P^b(a_i, z_{i-1}|u_i)P^b(u_i)}{\sum_{u_i'}P^b(a_i, z_{i-1}|u'_i)P^b(u'_i)},$$ we have that $$P^b(U_i|a_i=0, Z_{i-1}) = \begin{pmatrix}
\frac{\epsilon \zeta_0 s_i}{\epsilon\zeta_0 s_i + \delta\zeta_1(1-s_i)} & \frac{\epsilon(1-\zeta_0)s_i}{\epsilon(1-\zeta_0)s_i + \delta(1-\zeta_1)(1-s_i)}\\
\frac{\delta\zeta_1(1-s_i)}{\epsilon\zeta_0 s_i + \delta\zeta_1(1-s_i)} & \frac{\delta(1-\zeta_1)(1-s_i)}{\epsilon(1-\zeta_0)s_i + \delta(1-\zeta_1)(1-s_i)}
\end{pmatrix}$$ and $$P^b(U_i|a_i=1, Z_{i-1}) = \begin{pmatrix}
\frac{(1-\epsilon) \zeta_0 s_i}{(1-\epsilon)\zeta_0 s_i + (1-\delta)\zeta_1(1-s_i)} & \frac{(1-\epsilon)(1-\zeta_0)s_i}{(1-\epsilon)(1-\zeta_0)s_i + (1-\delta)(1-\zeta_1)(1-s_i)}\\
\frac{(1-\delta)\zeta_1(1-s_i)}{(1-\epsilon)\zeta_0 s_i + (1-\delta)\zeta_1(1-s_i)} & \frac{(1-\delta)(1-\zeta_1)(1-s_i)}{(1-\epsilon)(1-\zeta_0)s_i + (1-\delta)(1-\zeta_1)(1-s_i)}
\end{pmatrix}.$$ 

A sufficient condition for $\rank(P^b(U_i|a_i=0, Z_{i-1})) = \rank(P^b(U_i|a_i=1, Z_{i-1})) = 1 < 2$ is that $\zeta_0 = \zeta_1$. We now write $P^b(Z_{i-1}|U_i)$ in terms of the other parameters. 

Using the fact that $$P^b(u_{i-1}, a_{i-1}|u_i) = \frac{P^b(u_i|u_{i-1}, a_{i-1})\pi_b^{(i-1)}(a_{i-1}|u_{i-1})P^b(u_{i-1})}{\sum_{a_{i-1}', u_{i-1}'}P^b(u_i|u'_{i-1}, a'_{i-1})\pi_b^{(i-1)}(a'_{i-1}|u'_{i-1})P^b(u'_{i-1})}$$ we have that $$P^b(U_{i-1}, a_{i-1} = 0|U_i) = \begin{pmatrix}
\frac{a\epsilon s_{i-1}}{a\epsilon s_{i-1} + b(1-\epsilon)s_{i-1} + c\delta(1-s_{i-1}) + d(1-\delta)(1-s_{i-1})} & \frac{(1-a)\epsilon s_{i-1}}{1-(a\epsilon s_{i-1} + b(1-\epsilon)s_{i-1} + c\delta(1-s_{i-1}) + d(1-\delta)(1-s_{i-1}))}\\
\frac{c\delta(1-s_{i-1})}{a\epsilon s_{i-1} + b(1-\epsilon)s_{i-1} + c\delta(1-s_{i-1}) + d(1-\delta)(1-s_{i-1})} & \frac{(1-c)\delta (1-s_{i-1})}{1-(a\epsilon s_{i-1} + b(1-\epsilon)s_{i-1} + c\delta(1-s_{i-1}) + d(1-\delta)(1-s_{i-1}))}
\end{pmatrix}$$ and $$P^b(U_{i-1}, a_{i-1} = 1|U_i) = \begin{pmatrix}
\frac{b(1-\epsilon)s_{i-1}}{a\epsilon s_{i-1} + b(1-\epsilon)s_{i-1} + c\delta(1-s_{i-1}) + d(1-\delta)(1-s_{i-1})} & \frac{(1-b)(1-\epsilon) s_{i-1}}{1-(a\epsilon s_{i-1} + b(1-\epsilon)s_{i-1} + c\delta(1-s_{i-1}) + d(1-\delta)(1-s_{i-1}))}\\
\frac{d(1-\delta)(1-s_{i-1})}{a\epsilon s_{i-1} + b(1-\epsilon)s_{i-1} + c\delta(1-s_{i-1}) + d(1-\delta)(1-s_{i-1})} & \frac{(1-d)(1-\delta) (1-s_{i-1})}{1-(a\epsilon s_{i-1} + b(1-\epsilon)s_{i-1} + c\delta(1-s_{i-1}) + d(1-\delta)(1-s_{i-1}))}
\end{pmatrix}.$$ Marginalizing across $a_{i-1}$ gives that $$P^b(U_{i-1}|U_i) = \begin{pmatrix}
\frac{a\epsilon s_{i-1} + b(1-\epsilon)s_{i-1}}{a\epsilon s_{i-1} + b(1-\epsilon)s_{i-1} + c\delta(1-s_{i-1}) + d(1-\delta)(1-s_{i-1})} & \frac{(1-a)\epsilon s_{i-1} + (1-b)(1-\epsilon) s_{i-1}}{1-(a\epsilon s_{i-1} + b(1-\epsilon)s_{i-1} + c\delta(1-s_{i-1}) + d(1-\delta)(1-s_{i-1}))}\\
\frac{c\delta(1-s_{i-1}) + d(1-\delta)(1-s_{i-1})}{a\epsilon s_{i-1} + b(1-\epsilon)s_{i-1} + c\delta(1-s_{i-1}) + d(1-\delta)(1-s_{i-1})} & \frac{(1-c)\delta (1-s_{i-1}) + (1-d)(1-\delta) (1-s_{i-1})}{1-(a\epsilon s_{i-1} + b(1-\epsilon)s_{i-1} + c\delta(1-s_{i-1}) + d(1-\delta)(1-s_{i-1}))}
\end{pmatrix}$$

Finally, using the fact that $$P^b(z_{i-1}|u_i) = \sum_{u_{i-1}'}P^b(z_{i-1}|u_{i-1}')P^b(u_{i-1}'|u_i),$$ we have that $$\zeta_0 = \frac{\rho_0(a\epsilon s_{i-1} + b(1-\epsilon)s_{i-1}) + \rho_1(c\delta(1-s_{i-1}) + d(1-\delta)(1-s_{i-1}))}{a\epsilon s_{i-1} + b(1-\epsilon)s_{i-1} + c\delta(1-s_{i-1}) + d(1-\delta)(1-s_{i-1})}$$ and $$\zeta_1 = \frac{\rho_0((1-a)\epsilon s_{i-1} + (1-b)(1-\epsilon)s_{i-1}) + \rho_1((1-c)\delta(1-s_{i-1}) + (1-d)(1-\delta)(1-s_{i-1}))}{1 - (a\epsilon s_{i-1} + b(1-\epsilon)s_{i-1} + c\delta(1-s_{i-1}) + d(1-\delta)(1-s_{i-1}))}.$$ Finally,  \begin{multline}\label{code-equation}\zeta_0 = \zeta_1 \impliedby (\rho_0 s_{i-1} + \rho_1(1-s_{i-1}))(a\epsilon s_{i-1} + b(1-\epsilon)s_{i-1} + c\delta(1-s_{i-1}) + d(1-\delta)(1-s_{i-1})) \\ = \rho_0(a\epsilon s_{i-1} + b(1-\epsilon)s_{i-1}) + \rho_1(c\delta(1-s_{i-1}) + d(1-\delta)(1-s_{i-1})).\end{multline} Thus, for Experiment 2, we chose a setting of parameters so as to satisfy equation \eqref{code-equation}, thus violating Assumption \ref{invertibility}.

\subsection{Coarsening the Support in Experiment 3}
As mentioned in Remark \ref{noisy-remark}, one way of dealing with noisy estimates is by coarsening the support of the future distribution. This is done in our third experiment by taking the future, $f_i$, of $u_i$ to consist only of $a_{i+1}$. That is, the set $\mathcal{F}^o_i$ is simply $\mathcal{A}_{i+1}$. This has the benefit of having the same cardinality as $|\mathcal{U}|$ and thus allowing us to bypass identification of a strict subset of $|\mathcal{U}|$ (nearly) linearly independent columns. Additionally, the violation of Assumption \ref{history-rank} in Experiment 3---via way of making $\rank(P^b(Z_i|U_i)) = 1$ as mentioned in Appendix \ref{violate}---does not render $\rank(P^b(A_{i+1}|U_i)) = 1$, thus allowing for the use of this choice of future in Experiment 3 to guarantee that $\rank(P^b(\mathcal{F}^o_i|a_i, \mathcal{H}^o_{i-1})) = |\mathcal{U}|$.

\subsection{Adaptive Search in Experiment 3}\label{adaptive-search}
In the adaptive search conducted in Experiment 3, we first generate the POMDP parameters uniformly at random until the matrices $P^{\tau}_{i,a_i}$ have sufficiently low condition number and the observable product, $(P^\tau_{i,a_i})^+Q^\tau_{i,a_i,r_i}$, to be diagonalized has sufficiently low eigenvector condition number, which, in the case of $2\times 2$ matrices is simply the inverse of the spectral gap \citep{stewart}. After such a setting of parameters has been identified, we then again randomly generate parameters from a uniform distribution biased in the direction of what appears to be the optimal setting, with stricter stopping conditions on the condition numbers. For example, if, after the first round, one of the probability parameter settings is $0.9$, we then draw that parameter from $\text{Unif}[0.9,1]$ in the next round so as to find a setting of parameters which yields even better conditioned matrices.

\subsection{Robustness of Theorem \ref{history-thm} in Environments Violating Assumption \ref{history-rank}}
The setup for this experiment is essentially the same as that of Experiment 1 described in Section \ref{sec:experiments}, wherein the residual of each estimate of $P^e(r_t = 1)$ is averaged across $50$ trials. Figure \ref{fig:future_exp2} shows results for the estimators of Theorems \ref{history-thm} and \ref{future_thm}, where the average is over environments violating Theorem \ref{history-thm}'s Assumption \ref{history-rank}. In particular, note that, even though the environments violate Assumption \ref{history-rank}, the average residuals produced by Theorem \ref{history-thm}'s estimators are very low, indicating the high-quality estimates produced by Theorem \ref{history-thm} even in environments which violate its assumptions. Theorem \ref{future_thm}, however, produces, on average, much larger residuals. This is in contrast to Experiment 3 in Figure \ref{fig:future_exp}, wherein Theorem \ref{future_thm}'s estimator performs much better---in fact, even better than Theorem \ref{history-thm}'s. The reason for the diminished quality of Theorem \ref{future_thm} in Experiment 4 is due to the fact that, on average, the matrices that need to be inverted and diagonalized have extremely poor matrix and eigenvector condition numbers, resulting in severe degradation of Theorem \ref{future_thm}, as mentioned in Section \ref{sec:experiments}, even though its assumptions are not explicitly violated.

\begin{figure}
    \centering
    \includegraphics[scale=0.52]{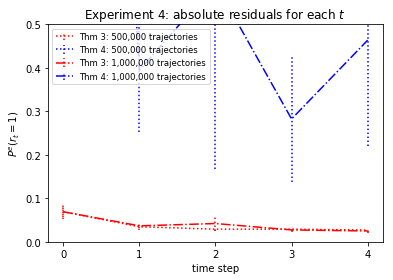}
    \caption{Average absolute residuals $|\hat{P}^e(r_t) -P^e(r_t)|$ over randomly generated POMDP environments which violate Assumption \ref{history-rank}.}
    \label{fig:future_exp2}
\end{figure}

\subsection{Choice of Parameters in Experiments 3 and 4}\label{violate}
In Experiments 3 and 4, in order to violate Assumption \ref{history-rank}, we simply chose POMDPs in which $\rank(P^b(Z_i|U_i)) < |\mathcal{U}|$. As $|\mathcal{Z}| = |\mathcal{U}| = 2$ in our experiments, this meant that we chose $$P^b(Z_i|U_i) = \begin{pmatrix}
\rho & \rho\\
1-\rho & 1-\rho
\end{pmatrix},$$ for $\rho \in [0,1]$.

\subsection{Importance Sampling Experiments and Results}
\begin{figure}[h]
\begin{center}
\begin{tabular}{||c c c c||} 
 \hline
 Number of trajectories & True value $(\pm \text{standard error})$ & Theorem \ref{is-thm} estimated value & Baseline IS estimated value \\ [0.5ex] 
 \hline\hline
 $1{,}000{,}000$ & $1{.}119 \pm 0.002$ & $157{.}639$ & $\mathbf{1{.}430}$ \\ 
 \hline
  $5{,}000{,}000$ & $1{.}119 \pm 0.002$ & $\mathbf{1{.}116}$ & $1{.}445$ \\ [1ex] 
 \hline
\end{tabular}
\end{center}
\caption{Estimated values compared to true value of estimators of Theorem \ref{is-thm} and the baseline for horizon $H = 1$.}
\label{fig:is1}
\end{figure}

\begin{figure}[h]
\begin{center}
\begin{tabular}{||c c c c||} 
 \hline
 Number of trajectories & True value $(\pm \text{standard error})$ & Theorem \ref{is-thm} estimated value & Baseline IS estimated value \\ [0.5ex] 
 \hline\hline
 $1{,}000{,}000$ & $1{.}863 \pm 0.002$ & $309{.}309$ & $\mathbf{2.314}$ \\ 
 \hline
 $5{,}000{,}000$ & $1{.}863 \pm 0.002$ & $2{.}878$ & $\mathbf{2{.}166}$ \\ [1ex] 
 \hline
\end{tabular}
\end{center}
\caption{Estimated values compared to true value of estimators of Theorem \ref{is-thm} and the baseline for horizon $H = 2$.}
\label{fig:is2}
\end{figure}

We also conducted experiments comparing our IS algorithm to that of \citet{Tennenholtz_Shalit_Mannor_2020} in an environment---again, chosen through adaptive search, similar to that described in Appendix \ref{adaptive-search}, so as to produce well-conditioned matrices---meeting the assumptions of Theorem \ref{is-thm} but not the sufficiency conditions of the baseline IS estimator of \citet{Tennenholtz_Shalit_Mannor_2020}. Figures \ref{fig:is1} and \ref{fig:is2} shows the results of the two estimators compared at both $1{,}000{,}000$ and $5{,}000{,}000$ number of sampled trajectories at horizons $H = 1$ and $H = 2$. In Figure \ref{fig:is1}, we see that at horizon $1$ (i.e., two time-steps), our estimator produces very inaccurate results for $1{,}000{,}000$ sampled trajectories, but is much more accurate with $5{,}000{,}000$ trajectories than even the baseline, which is known to be biased and inconsistent.

However, for $H = 2$, it appears that even $5{,}000{,}000$ number of trajectories is not enough to beat the baseline IS estimator of \citet{Tennenholtz_Shalit_Mannor_2020}. While the quality of Theorem \ref{is-thm}'s estimator improves greatly from $1{,}000{,}000$ to $5{,}000{,}000$ number of trajectories, Figure \ref{fig:is2} illustrates that our estimator requires even more trajectories in order to produce more accurate results than the biased and inconsistent baseline.

\section{OPE via Importance Sampling}\label{is-proof}
We now give a derivation for the IS estimator given in Section 6. As shown in the proof sketch of Theorem 6, we can write \begin{equation}v_H(\pi_e) =  \mathbb{E}[vW(v,\tau^o)]\end{equation} with \begin{equation}W(v,\tau^o) := \frac{1}{P^b(v|\tau^o)}\Pi_e(\tau^o)\Gamma_b(v,\tau^o)\end{equation} and \begin{equation}\Gamma_b(v,\tau^o) = \sum_{r_{0:H}: \sum r_i = v}\sum_{u_{0:H}}\frac{\prod_{i=0}^HP^b(r_i|u_i)P^b(u_{0:H}|\tau^o)}{\Pi_b(u_{0:H})}.\end{equation} Hence, all that must be shown is the identifiability of $\Gamma_b(v,\tau^o)$, as the terms of $\frac{1}{P^b(v|\tau^o)}\Pi_e(\tau^o)$ are all either given or directly estimable from observable data.

\subsection{Identifiability of $\Gamma_b(v,\tau^o)$}\label{is-identification-proof}
First, we show that under certain rank and distinctness assumptions, the vectors $\pi_b^{(i)}(a_i|U_i)$ and $P^b(r_i|U_i)$ are identifiable by extending the diagonalization method of \citet{kuroki}. We will treat both of the above vectors individually. Again, letting $\kappa = |\mathcal{U}|$, we first define a matrix which will be used in both derivations: $$\Delta_{i,z_{i+1}} := \diag(P^b(z_{i+1}| u_i^{(1)}), \ldots, P^b(z_{i+1}|u_i^{(\kappa)}).$$ The ordering $u_i^{(j)}$ is such that the elements on the diagonal of $\Delta_{i,z_{i+1}}$ are in non-decreasing order. That is, $P^b(z_{i+1}| u_i^{(1)}) \geq \ldots \geq P^b(z_{i+1}|u_i^{(\kappa)}))$. Furthermore, we will let $z^j, r^j$, and $a^j$ denote the $j$th elements of the sets $\mathcal{Z}$, $\mathcal{R}$, and $\mathcal{A}$, respectively. Henceforth, the vectors $\pi_b^{(i)}(a_i|U_i)$ and $P^b(z_i|U_i)$ are defined such that their $j$th entries are $\pi_b^{(i)}(a_i|u^{(j)}_i)$ and $P^b(z_i|u^{(j)}_i)$, respectively. Throughout our identification procedure for both of these vectors, we make the following crucial distinctness assumption, which will allow for unique eigendecomposition (up to constants):

\begin{assumption}[Distinctness]\label{appendix_distinctness}
For each $z_{i+1} \in \mathcal{Z}$, the probabilities $\{P^b(z_{i+1}|u_i): u_i \in \mathcal{U}\}$ are all distinct.
\end{assumption}

\subsubsection{Identifying $\pi_b^{(i)}(a_i|U_i)$}\label{policy}
Again, similar to those defined by \citet{kuroki}, we first define the following matrices which will be key to the analysis, explicitly state all assumptions in terms of them, and then show the result. 

Define the matrix following matrices $$P^a_{i,z_i} := \begin{pmatrix}
1 & P^b(a_i^1|z_i) & \cdots & P^b(a_i^{|\mathcal{A}|-1}|z_i)\\
P^b(z_{i-1}^1|z_i) & P^b(a_i^1, z_{i-1}^1|z_i) & \cdots & P^b(a_i^{|\mathcal{A}|-1}, z_{i-1}^1|z_i)\\
\vdots & \vdots & \ddots & \vdots\\
P^b(z_{i-1}^{|\mathcal{Z}|-1}|z_i) & P^b(z_{i-1}^{|\mathcal{Z}|-1}, a_i^1|z_i) & \cdots & P^b(a_i^{|\mathcal{A}|-1}, z_{i-1}^{|\mathcal{Z}|-1}|z_i)
\end{pmatrix},$$  $$Q^a_{i, z_i, z_{i+1}} := \begin{pmatrix}
P^b(z_{i+1}|z_i) & P^b(a_i^{1}, z_{i+1}|z_i) & \cdots & P^b(a_i^{|\mathcal{A}|-1}, z_{i+1}|z_i)\\
P^b(z_{i-1}^1, z_{i+1}|z_i) & P^b(a_i^1, z_{i-1}^1, z_{i+1}|z_i) & \cdots & P^b(a_i^{|\mathcal{A}|-1}, z_{i-1}^1, z_{i+1}|z_i)\\
\vdots & \vdots & \ddots & \vdots\\
P^b(z_{i-1}^{|\mathcal{Z}|-1}, z_{i+1}|z_i) & P^b(a_i^1, z_{i-1}^{|\mathcal{Z}|-1}, z_{i+1}|z_i) & \cdots & P^b(a_i^{|\mathcal{A}|-1}, z_{i-1}^{|\mathcal{Z}|-1}, z_{i+1}|z_i)
\end{pmatrix},$$ $$U^a_i := \begin{pmatrix}
1 & \pi_b^{(i)}(a_i^1|u_i^{(1)}) & \cdots & \pi_b^{(i)}(a_i^{|\mathcal{A}|-1}|u_i^{(1)})\\
\vdots & \vdots & \ddots & \vdots\\
1 & \pi^{(i)}_b(a_i^1|u_i^{(\kappa)}) & \cdots & \pi_b^{(i)}(a_i^{|\mathcal{A}|-1}|u_i^{(\kappa)})
\end{pmatrix},$$ $$R^a_{i, z_i} = \begin{pmatrix}
1 & P^b(z_{i-1}^1|z_i, u_i^{(1)}) & \cdots & P^b(z_{i-1}^{|\mathcal{Z}|-1}|z_i, u_i^{(1)})\\
\vdots & \vdots & \ddots & \vdots\\
1 & P^b(z_{i-1}^1|z_i, u_i^{(\kappa)}) & \cdots & P^b(z_{i-1}^{|\mathcal{Z}|-1}|z_i, u_i^{(\kappa)})
\end{pmatrix},$$ $$M^a_{i,z_i} := \diag(P^b(u_i^{(1)}|z_i), \ldots, P^b(u_i^{(\kappa)}|z_i)).$$

In addition to Assumptions \ref{necessary}, \ref{reward}, and \ref{appendix_distinctness}, we also make the following rank assumption:

\begin{assumption}[Rank]\label{appendix_rank}
The matrices $P^a_{i,z_i}$ and $Q^a_{i,z_i}$ have rank at least $|\mathcal{U}|$.
\end{assumption}

With these conditions, we now show how to obtain the matrix $U^a_i$, which contains the desired probabilities, via diagonalization of observable matrices.

Again, we have that \begin{equation}\label{Pa}
    P^a_{i,z_i} = R^a_{i,z_i}{}^TM^a_{i,z_i}U^a_i\end{equation} and 
    \begin{equation}\label{Qa}Q^a_{i,z_i,z_{i+1}} = R^a_{i,z_i}{}^TM^a_{i,z_i}\Delta_{i,z_{i+1}}U^a_i
\end{equation}

The identification of $U^a_i$ from equations \eqref{Pa} and \eqref{Qa} is then essentially the same as that of $P^b(\mathcal{F}^o_{i+1}, z_i|a_i, \mathcal{H}^o_{i-1}||a_{i+1})$ in Appendix \ref{future-proof}, wherein we can either identify $|\mathcal{U}|$ independent columns of $P^a_{i,z_i}$ (by looking at condition numbers of the noisey estimates of submatrices of $P^a_{i,z_i}$) or coarsen the support of $|\mathcal{A}|$ to make $P^a_{i,z_i}, Q^a_{i,z_i,z_{i+1}}$, and $U^a_i$ have second dimension equal to $|\mathcal{U}|$ and then employ the eigendecomposition strategy of \citet{kuroki} by diagonalizing $(P^a_{i,z_i}L(\Lambda))^+Q^a_{i,z_i,z_{i+1}}L(\Lambda)$ in the former case or $(P^a_{i,z_i})^+Q^a_{i,z_i,z_{i+1}}$ in the latter, and renormalizing by the appropriate constants to make the first column of the left eigenvector matrix equal to $\vec{1}$. 

\subsubsection{Identifying $P^b(r_i|U_i)$}
The identification procedure for $P^b(r_i|U_i)$ is essentially the same as that in the previous section, except that we use the following matrices:

$$P^r_{i,a_i, z_i} = \begin{pmatrix}
1 & P^b(r_i^1|z_i) & \cdots & P^b(r_i^{|\mathcal{R}|-1}|z_i)\\
P^b(z_{i-1}^1, a_i|z_i) & P^b(r_i^1, z_{i-1}^1, a_i|z_i) & \cdots & P^b(r_i^{|\mathcal{R}|-1}, z_{i-1}^1, a_i|z_i)\\
\vdots & \vdots & \ddots & \vdots\\
P^b(z_{i-1}^{|\mathcal{Z}|-1}, a_i|z_i) & P^b(r_i^{1}, z_{i-1}^{|\mathcal{Z}|-1}, a_i|z_i) & \cdots & P^b(r_i^{|\mathcal{R}|-1}, z_{i-1}^{|\mathcal{Z}|-1}, a_i|z_i)
\end{pmatrix},$$ $$Q^r_{i,z_i, z_{i+1}, a_i} = \begin{pmatrix}
P^b(z_{i+1}|z_i) & P^b(r_i^{1}, z_{i+1}|z_i) & \cdots & P^b(r_i^{|\mathcal{R}|-1}, z_{i+1}|z_i)\\
P^b(z_{i-1}^1, z_{i+1}, a_i|z_i) & P^b(r_i^1, z_{i-1}^1, z_{i+1}, a_i|z_i) & \cdots & P^b(r_i^{|\mathcal{R}|-1}, z_{i-1}^1, z_{i+1}, a_i|z_i)\\
\vdots & \vdots & \ddots & \vdots\\
P^b(z_{i-1}^{|\mathcal{Z}|-1}, z_{i+1}, a_i|z_i) & P^b(r_i^1, z_{i-1}^{|\mathcal{Z}|-1}, z_{i+1}, a_i|z_i) & \cdots & P^b(r_i^{|\mathcal{R}|-1}, z_{i-1}^{|\mathcal{Z}|-1}, z_{i+1}, a_i|z_i)
\end{pmatrix},$$ $$U^r_i = \begin{pmatrix}
1 & P^b(r_i^1|u_i^{(1)}) & \cdots & P^b(r_i^{|\mathcal{R}|-1}|u_i^{(1)}\\
\vdots & \vdots & \ddots & \vdots\\
1 & P^b(r_i^1|u_i^{(k)}) & \cdots & P^b(r_i^{|\mathcal{R}|-1}|u_i^{(k)})
\end{pmatrix},$$ $$R^r_{i, z_i, a_i} = \begin{pmatrix}
1 & P^b(a_i, z_{i-1}^1|z_i, u_i^{(1)}) & \cdots & P^b(a_i, z_{i-1}^{|\mathcal{Z}|-1}|z_i, u_i^{(1)})\\
\vdots & \vdots & \ddots & \vdots\\
1 & P^b(a_i, z_{i-1}^1|z_i, u_i^{(k)}) & \cdots & P^b(a_i, z_{i-1}^{|\mathcal{Z}|-1}|z_i, u_i^{(k)})
\end{pmatrix},$$ and $$M^r_{i,z_i} = \diag(P^b(u_i^{(1)}|z_i), \ldots, P^b(u_i^{(k)}|z_i)).$$ Again, analogously to the previous section, in addition to Assumptions \ref{necessary}, \ref{reward}, and \ref{appendix_distinctness}, we also make the following rank assumption:

\begin{assumption}[Rank]\label{appendix_rank_reward}
The matrices $P^r_{i,z_i}$ and $Q^r_{i,z_i}$ have rank at least $|\mathcal{U}|$.
\end{assumption}

Importantly, we have the same identities on the above matrices:
\begin{equation}\label{Pr}
    P^r_{i,z_i} = R^r_{i,z_i}{}^TM^r_{i,z_i}U^a_i\end{equation} and 
    \begin{equation}\label{Qr}Q^r_{i,z_i,z_{i+1}} = R^r_{i,z_i}{}^TM^r_{i,z_i}\Delta_{i,z_{i+1}}U^r_i
\end{equation}

The procedure for identifying $U^r_i$ is then precisely the same as that described in the previous section and Appendix \ref{future-proof}. Importantly, because the diagonal matrix is the \textit{same} matrix $\Delta_{i,z_{i+1}}$ used to identify $\pi_b^{(i)}(a_i|U_i)$ and because of the distinctness prescribed by Assumption \ref{appendix_distinctness}, the ordering $u_i^{(1)}, \ldots, u_i^{(k)}$ is the same, so that when we identify the vector $P^b(r_i|U_i)$, the ordering of entries is indeed in corresponding order to our identification of $\pi_b^{(i)}(a_i|U_i)$.

We summarize the results from the previous $2$ sections here before we proceed with the rest of the identification procedure for $\Gamma_b(v,\tau^o)$.

\begin{theorem}
Under Assumptions \ref{necessary}, \ref{reward}, \ref{appendix_distinctness}, \ref{appendix_rank}, and \ref{appendix_rank_reward}, the vectors $\pi_b^{(i)}(a_i|U_i)$ and $P^b(r_i|U_i)$  are identifiable, under the consistent ordering that their $j$th entries are $\pi_b^{(i)}(a_i|u_i^{(j)})$ and $P^b(r_i|u_i^{(j)})$, respectively.
\end{theorem}

We now proceed with the identification procedure for $\Gamma_b(v,\tau^o)$. The sole remaining ingredient is the identification of $P^b(U_{0:H}|\mathcal{T}_H^o)$, which we now derive.

\subsubsection{Identifying $P^b(U_{0:H}|\mathcal{T}_H^o)$}
\begin{lemma}
Under Assumptions \ref{necessary}, \ref{reward}, and \ref{expo-rank-IS}, the matrix $P^b(U_{0:H}|\mathcal{T}^o_{\mathcal{H}})$ is identifiable.
\end{lemma}
\begin{proof}
Note that $P^b(R_{0:H}|U_{0:H}) = \bigotimes_{h=0}^HP^b(R_h|U_h)$. Thus $\rank(P^b(R_{0:H}|U_{0:H})) = |\mathcal{U}|^{H+1}$ by Assumption \ref{expo-rank-IS}, and thus $P^b(R_{0:H}|U_{0:H})$ has full column rank. The identifiability of $P^b(R_h|U_h)$ (up to the unknown ordering) implies the identifiability of $P^b(R_{0:H}|U_{0:H})$ (also up to the unknown ordering). Finally, note that $$P^b(R_{0:H}|U_{0:H})P^b(U_{0:H}|\mathcal{T}^o_{\mathcal{H}}) = P^b(R_{0:H}|\mathcal{T}^o_{\mathcal{H}})$$ $$\implies P^b(U_{0:H}|\mathcal{T}^o_{\mathcal{H}}) = P^b(R_{0:H}|U_{0:H})^+P^b(R_{0:H}|\mathcal{T}^o_{\mathcal{H}}),$$ thus rendering $P^b(U_{0:H}|\mathcal{T}^o_{\mathcal{H}})$ identifiable, as $P^b(R_{0:H}|\mathcal{T}^o_{\mathcal{H}})$ is observable.
\end{proof}

\subsubsection{Combining Identifications}

Finally, recalling that \begin{equation*}\Gamma_b(v,\tau^o) = \sum_{r_{0:H}: \sum r_i = v}\sum_{u_{0:H}}\frac{\prod_{i=0}^HP^b(r_i|u_i)P^b(u_{0:H}|\tau^o)}{\Pi_b(u_{0:H})},\end{equation*} \begin{equation*} = \sum_{r_{0:H}: \sum r_i = v}\sum_{1 \leq j_0, \ldots, j_H \leq \kappa}\frac{\prod_{i=0}^HP^b(r_i|u^{(j_i)}_i)P^b(u_0^{(j_0)}, \ldots, u_H^{(j_H)}|\tau^o)}{\Pi_b(u_0^{(j_0)}, \ldots, u_H^{(j_H}))},\end{equation*} the above derivations allow for identification for \textit{each} of the terms $P^b(r_i|u^{(j_i)}_i)$, $P^b(u_0^{(j_0)}, \ldots, u_H^{(j_H)}|\tau^o)$, and $\Pi_b(u_0^{(j_0)}, \ldots, u_H^{(j_H}))$, thus implying the desired identifiability for the entire IS estimation procedure, as desired.

We summarize this result and all conditions below:

\begin{theorem}
Under Assumptions \ref{necessary}, \ref{reward}, \ref{expo-rank-IS}, \ref{main-alg-pos}, \ref{appendix_distinctness}, \ref{appendix_rank}, \ref{appendix_rank_reward}, the quantity $\Gamma_b(v,\tau^o)$ is identifiable for all values $v$ and $\tau^o \in \mathcal{T}^o_H$. Hence, the importance weights \begin{equation*}W(v,\tau^o) := \frac{1}{P^b(v|\tau^o)}\Pi_e(\tau^o)\Gamma_b(v,\tau^o)\end{equation*} are all identifiable, allowing for the IS procedure given by $v_H(\pi_e) =  \mathbb{E}[vW(v,\tau^o)]$
\end{theorem}

\end{document}